%% file: main.tex
\newif\ifarxiv
    \arxivtrue % for arxiv full version
\ifarxiv
\documentclass[a4paper]{article}
\usepackage[a4paper,top=3.5cm,bottom=3cm,left=3.5cm,right=3.5cm,marginparwidth=1.5cm]{geometry}
\usepackage{amsthm}
\usepackage{amstext,amsmath}
\usepackage{mathtools}

\usepackage[utf8]{inputenc}
\usepackage{bbm}
\usepackage{times}
\usepackage{graphicx,color}
\usepackage{array,float}
\usepackage{url}
\usepackage{hyphenat}
\usepackage{verbatim}
\usepackage{bm}
\usepackage{paralist}
\usepackage{ulem}

\usepackage{paralist}\usepackage{bbm}
\usepackage{wrapfig}
\usepackage{hyperref}
\usepackage[noend]{algorithmic}
\usepackage{algorithm}
\usepackage{xparse,etoolbox}
\usepackage{threeparttable}
\usepackage{dsfont}

\else

\documentclass[review,onefignum,onetabnum]{siamart190516}

\fi

\usepackage{amssymb} 
\usepackage{mathabx}
\usepackage{enumitem}

\newcommand{\cO}{\ensuremath{\mathcal{O}}}
\newcommand{\cX}{\ensuremath{\mathcal{X}}}

\newcommand{\cZ}{\ensuremath{\mathcal{Z}}}	
\newcommand{\cD}{\mathcal{D}}
\newcommand{\cA}{\mathcal{A}}

\newcommand{\cR}{\mathcal{R}}

\newcommand{\cV}{\mathcal{V}}

\newcommand{\cF}{\mathcal{F}}
\newcommand{\cM}{\mathcal{M}}
\newcommand{\cT}{\mathcal{T}}

\newcommand{\cGG}{\mathcal{GG}}
\newcommand{\bfd}{\mathbf{d}}

\newcommand{\bDel}{\widebar{\Delta}}
\newcommand{\tDel}{\widetilde{\Delta}}

\newcommand{\hx}{\widehat{x}}
\newcommand{\hs}{\hat{s}}
\newcommand{\xs}{x^{\ast}}

\newcommand{\brz}{\bar{z}}
\newcommand{\brv}{\bar{v}}

\newcommand{\tnab}{\widetilde{\nabla}}

\newcommand{\hS}{\widehat{S}}

\newcommand{\lap}{\mathsf{Lap}}
\newcommand{\prv}{\mathsf{priv}}

\newcommand{\dfs}{\mathsf{DFS}}

\newcommand{\scg}{\mathsf{noisySFW}}

\newcommand{\polyfw}{\mathsf{polySFW}}
\newcommand{\nsgd}{\mathsf{noisySGD}}

\DeclareMathOperator*{\argmin}{arg\,min}

\newcommand{\ip}[2]{\langle #1,#2\rangle}

\newcommand{\mypar}[1]{\vspace{6pt} {\it {#1}}}

\newcommand{\pr}[2]{\underset{#1}{\mathbb{P}}\left[ #2 \right]}
\newcommand{\ex}[2]{\underset{#1}{\mathbb{E}}\left[ #2 \right]}

\newcommand{\norm}[1]{\|#1\|}
\newcommand{\dual}[1]{\|#1\|_*}

\newcommand{\G}{\Gamma}
\newcommand{\R}{\mathbb{R}}
\newcommand{\PP}{\mathbb{P}}
\newcommand{\E}{\mathbb{E}}
\newcommand{\B}{\mathcal{B}}

\newcommand{\bB}{\mathbf{B}}
\newcommand{\bd}{\mathbf{d}}
\newcommand{\bg}{\mathbf{g}}

\newcommand{\mg}{\mathbf{n}}
\newcommand{\regnorm}[1]{\|#1\|}

\newcommand{\bE}{\mathbf{E}}

\newcommand{\Risk}{{\cal R}}
\newcommand{\Hp}{{\cal E}}
\newcommand{\kul}{\mathsf{RndSmth}}

\newif\ifnotes
 \notesfalse

\ifnotes
\usepackage{todonotes}

\newcommand{\rnote}[1]{ [\textcolor{cyan}{Raef: #1}] }
\newcommand{\anote}[1]{ [\textcolor{purple}{Anupama: #1}] }
\newcommand{\cnote}[1]{ [\textcolor{blue}{Cristobal: #1}] }

\else

\newcommand{\rnote}[1]{}
\newcommand{\anote}[1]{}
\newcommand{\cnote}[1]{}
\fi

\ifarxiv
\newtheorem{lem}{Lemma}[section]
\newtheorem{thm}{Theorem}
\newtheorem{remark}[thm]{Remark}
\newtheorem{theorem}[lem]{Theorem}

\newtheorem{corollary}[lem]{Corollary}

\newtheorem{definition}[lem]{Definition}
\newtheorem{fact}[lem]{Fact}

\newtheorem{claim}[lem]{Claim}
\newtheorem{proposition}[lem]{Proposition}

\title{Non-Euclidean Differentially Private Stochastic Convex Optimization: Optimal Rates in Linear Time\thanks{This work is an extension of \cite{BGN:2021}.}}

\author{%
	Raef Bassily\thanks{Department of Computer Science \& Engineering, Translational Data Analytics Institute (TDAI), The Ohio State University. \texttt{bassily.1@osu.edu}} 
		\and  Crist\'obal Guzm\'an \thanks{Department of Applied Mathematics, University of Twente
        and Institute for Mathematical and Computational Engineering, Pontificia Universidad Cat\'olica de Chile  \texttt{c.guzman@utwente.nl}}
		\and Anupama Nandi \thanks{Department of Computer Science \& Engineering, The Ohio State University. \texttt{nandi.10@osu.edu}}
}
\date{}

\else

\newsiamthm{fact}{Fact}

\input{ex_shared}

\fi
\begin{document}

\maketitle
% REQUIRED
\begin{abstract} %\rnote{Modify the abstract}
Differentially private (DP) stochastic convex optimization (SCO) is a fundamental problem, where the goal is to approximately minimize the population risk with respect to a convex loss function, given a dataset of $n$ i.i.d. samples from a distribution, while satisfying differential privacy with respect to the dataset. Most of the existing works in the literature of private convex optimization focus on the Euclidean (i.e., $\ell_2$) setting, where the loss is assumed to be Lipschitz (and possibly smooth) w.r.t.~the $\ell_2$ norm over a constraint set with bounded $\ell_2$ diameter. Algorithms based on noisy  stochastic gradient descent (SGD) are known to attain the optimal excess risk in this setting. 

In this work, we conduct a systematic study of DP-SCO for $\ell_p$-setups under a standard smoothness assumption on the loss. For $1< p\leq 2$, under a standard smoothness assumption, we give a new, linear-time DP-SCO algorithm with optimal excess risk. Previously known constructions with optimal excess risk for $1< p <2$ run in super-linear time in $n$. For $p=1$, we give an algorithm with nearly optimal excess risk. Our result for the $\ell_1$-setup also extends to general polyhedral norms and feasible sets. Moreover, we show that the excess risk bounds resulting from our algorithms for $1\leq p \leq 2$ are attained with high probability. For $2 < p \leq \infty$, we show that existing linear-time constructions for the Euclidean setup attain a nearly optimal excess risk in the low-dimensional regime. As a consequence, we show that such constructions attain a nearly optimal excess risk for $p=\infty$. Our work draws upon concepts from the geometry of normed spaces, such as the notions of regularity, uniform convexity, and uniform smoothness.

    % Differentially private (DP) stochastic convex optimization (SCO) is a fundamental problem, where the goal is to approximately minimize the population risk with respect to a convex loss function, given a dataset of i.i.d. samples from a distribution, while satisfying differential privacy with respect to the dataset. Most of the existing works in the literature of private convex optimization focus on the Euclidean (i.e., $\ell_2$) setting, where 
    % the loss is assumed to be Lipschitz (and possibly smooth) w.r.t.~the $\ell_2$ norm over a constraint set with bounded $\ell_2$ diameter.
    % Algorithms based on noisy  stochastic gradient descent (SGD) are known to attain the optimal excess risk in this setting.
    
    % In this work, we conduct a systematic study of DP-SCO for $\ell_p$-setups.
      
    % For $p=1$, under a standard smoothness assumption, we give a new algorithm with nearly optimal excess risk. This result also extends to general polyhedral norms and feasible sets. For $p\in(1, 2)$, we give two new algorithms, 
    % for which a central building block is a novel privacy mechanism, which generalizes the Gaussian mechanism. For $p\in (2, \infty)$, noisy SGD attains optimal excess risk in the low-dimensional regime; in particular, this proves the optimality of noisy SGD for $p=\infty$. Our work draws upon concepts from the geometry of normed spaces, such as the notions of regularity, uniform convexity, and uniform smoothness.
    
\end{abstract}

\ifarxiv
\else
% REQUIRED
\begin{keywords}
  Differential privacy, stochastic convex optimization, non-Euclidean norms.%
\end{keywords}

% REQUIRED
\begin{AMS}
	90C25,68Q25,68Q32,68W20
\end{AMS}
\fi

\input{intro}

\input{prelim}

\input{SCOp12}

\input{SCOPrivFW}

\input{SCO_infty}
\input{ack}

\bibliographystyle{siamplain}
\bibliography{references,reference,refs2,refs3}
%\newpage

\ifarxiv
\input{arxiv_appendix}

\else
\input{appendix}
\fi

\end{document}

%% file: intro.tex
\section{Introduction}

Stochastic Convex Optimization (SCO) is one of the most fundamental problems in optimization, statistics, and machine learning. In this problem, the goal is to minimize the expectation of a convex loss with respect to a distribution given a dataset of i.i.d.~samples from that distribution. A closely related problem is known as Empirical Risk Minimization (ERM), where the goal to minimize the empirical average of a loss with respect to a dataset (see Section~\ref{sec:prelim} for a more formal description). %\rnote{I don't think we need to talk about DP-ERM. This was for the COLT paper but now we removed the ERM result. We should focus only on DP-SCO.} \cnote{Should we just keep this? Since it is mentioned in a few parts.}

There has been a long line of works that studied the differentially private analogs of these problems known as DP-SCO and DP-ERM, %\rnote{same comment as above},
e.g., \cite{CM08,CMS,KST12,JTOpt13,BST,TTZ15,wang2017differentially,bassily2019private,feldman2020private}. Nevertheless, the existing theory does not capture a satisfactory understanding of private convex optimization in non-Euclidean settings, and particularly with respect to general $\ell_p$ norms. Almost all previous works that studied the general formulations of DP-ERM and DP-SCO under general convex losses %, with the exception of %\cite{TTZ15, TTZ15a}, 
focused on the \emph{Euclidean setting}, where both the constraint set and the subgradients of the loss are assumed to have a bounded $\ell_2$-norm. In this setting, algorithms with optimal error rates are known for DP-ERM \cite{BST} and DP-SCO \cite{bassily2019private, feldman2020private, bassily2020stability}. On the other hand, \cite{TTZ15,TTZ15a} is the only work we are aware of that studied non-Euclidean settings under a fairly general setup in the context of DP-ERM (see ``Other Related Work'' section below for other works that studied special cases of this problem). However, this work does not address DP-SCO; moreover, for $p>1$, it only provides upper bounds on the error rate for DP-ERM.

Without privacy constraints, convex optimization in these settings is fairly well-understood in the classical theory. In particular, there exists a universal algorithm that attains optimal rates for ERM as well as SCO over general $\ell_p$ spaces, namely, the stochastic mirror descent algorithm \cite{nemirovsky1983problem,nemirovskistochastic}. A key insight from this line of work is that the flexibility of non-Euclidean norms permits polynomial (in the dimension) acceleration for stochastic first-order methods (see, e.g., discussions in \cite[Sec.~5.1.1]{Sra:2011}). %\cnote{I modified the paragraph in line with DP results being already proved.}\rnote{looks fine. I just turned it to back color.}

% \cnote{I would remove this paragraph, and continue with the text in red in the next paragraph.}\rnote{Yes. I  removed the preceding text and made minor edits to the red paragraph.}
% \cnote{Please check the first 2 rows in the first column. I believe the current bounds are slightly different (regarding dependence on $\kappa$, $\log d$.}\rnote{Based on the rough analysis I had earlier $\kappa$ was outside, but the current guarantees in Thm 5.3 are tighter. @Anupama, could you check the proof steps again?} 
\begin{table}[H]
\centering
   \begin{tabular}{ |c|c|c| } 

 \hline
  $\ell_p$-setup &  Upper Bound &  Lower  bound \\
 \hline
 $p=1$ & $ \tilde O\Big(\frac{\log(d)}{\varepsilon \sqrt{n}}\Big)$ \,\,$(\ast)$ \quad &  \quad$\Omega\Big(\sqrt{\frac{\log d}{n}}\Big)$\qquad\, {\footnotesize(ABRW'12)}  \\ 
 \hline
 $1<p<2$ & $\tilde O\Big( \sqrt{\frac{\kappa}{n}}+ \frac{\kappa\sqrt{d}}{\varepsilon n} \Big)$  & $\Omega\left(\frac{1}{\sqrt{n}}+ \frac{(p-1)\sqrt{d}}{\varepsilon n}\right)$ \\ 
  \hline
 $2<p\leq \infty$ & $\tilde O\Big(\frac{d^{1/2-1/p}}{\sqrt{n}} + \frac{d^{1-1/p}}{\varepsilon n}\Big)$  & $\Omega\left(\min\left\{\frac{d^{1/2-1/p}}{\sqrt{n}},\frac{1}{n^{1/p}}\right\}\right)$ \footnotesize{(NY'83,ABRW'12)}  \\ 
 \hline
 $p=\infty$ & $O\Big(\sqrt{\frac{d}{n}}\Big)$ \quad\, & \quad\,\,$\Omega\Big(\sqrt{\frac{d}{n}}\Big)$ \quad\,\, {\footnotesize(ABRW'12)}  \\
 \hline
\end{tabular}
    \caption{\small{Bounds for excess risk of $(\varepsilon, \delta)$-DP-SCO. Here $d$ is dimension, $n$ is sample size,  %$\varepsilon$ is the main privacy parameter, 
    and $\kappa = \min\{1/(p-1),2\ln d\}$;  dependence on other parameters is omitted. $\tilde O(\cdot)$ hides polylogarithmic factors in $n$ and $1/\delta$. Existing lower bounds are for nonprivate SCO: NY'83 \cite{nemirovsky1983problem}, ABRW'12 \cite{Agarwal:2012}. $(\ast)$: Bound shown for $\ell_1$-ball feasible set.}}%Assumes feas.~set is $\ell_1$-ball.} }
    \label{tab:ex_risk_DP_SCO}
\end{table}
\vspace{-0.3cm}

% \cnote{I would replace this paragraph by what's in red below.}\rnote{Yes. I made minor edits.} \cnote{Looks fine.}
Recent work, including our conference paper \cite{AFKT:2021,BGN:2021}%\rnote{``our conference paper'' instead of ``conference version of this paper''}
, established nearly tight upper and lower bounds on the excess risk of DP-SCO in $\ell_p$ setups (see Table \ref{tab:ex_risk_DP_SCO}). Some of the major open questions left open in those works are whether there are linear-time (in the dataset size) algorithms for the case $1<p< 2$, and whether these algorithms directly provide high-probability guarantees for the optimal excess risk rather than the weaker expectation guarantees shown in the prior works.

%{\color{red}Raef: We should also elaborate on the differences between our analysis of the private algorithm and that of the non-private version including the fixed step size, the use of regular norms, etc.}

\subsection{Overview of Results}

 %Our work is the first to address stochastic convex optimization beyond Euclidean setups. 
 We formally study DP-SCO beyond Euclidean setups. More importantly, we identify the appropriate structures that suffice to attain nearly optimal rates in these settings. A crucial ingredient of our algorithms and lower bounds are the concepts of uniform convexity and uniform smoothness in a normed space. More concretely, we use the notion of {\em $\kappa$-regularity} of a normed space \cite{Juditsky:2008}, which quantifies how (close to) smooth is its squared norm (see Section~\ref{sec:prelim} for a formal definition). This concept has been applied in (nonprivate) convex optimization to design strongly convex regularizers, and to bound the deviations of independent sums and martingales in normed spaces. In this work, we make use of these ideas and we further show that $\kappa$-regular spaces have a natural noise addition DP mechanism that we call the {\em generalized Gaussian mechanism} (see Section~\ref{sec:gen-gauss}). We remark that this mechanism may be of independent interest.

\noindent Now, we focus on $\ell_p$-setups and describe our results for the different  values of $1\leq p\leq \infty$:
%\rnote{I switched the order stating $p\in (1, 2]$ first. We can discuss this if you feel it's better to keep it as it was. I feel it's better to discuss the new results first.}

\mypar{Case of $1<p\leq 2$:} %\rnote{replaced $<2$ with $\leq 2$}  
This regime of $p$ is interesting to investigate given the fact that the dependence of the optimal excess risk on the dimension of the problem is logarithmic in the $\ell_1$-setup \cite{BGN:2021, Asi:2021} and polynomial in the Euclidean setup \cite{bassily2019private}. %It is interesting to investigate the risk of DP-SCO between $p=1$ and $p=2$. This question is intriguing since for $p=1$ we have shown nearly dimension-independent rates, whereas for $p=2$ it is known optimal rates grow polynomially with $d$. 
Our work provides nearly tight upper and lower bounds on the excess risk of DP-SCO in this regime. Our lower bound shows a surprising phenomenon: there is a {\em sharp phase transition} of the excess risk around $p=1$. In fact, when $1+\Omega(1)<p<2$,
our lower bound is essentially the same as those of the $\ell_2$ case. This shows that the dependence on $\sqrt{d}$ is necessary in this regime, thus solving an open question posed in \cite{TTZ15a}. 
Our upper bound is attained via a new, linear-time algorithm. The best known previous construction for this problem for $1<p<2$ \cite{AFKT:2021} runs in super-linear, namely,  $O(\min\{n^{3/2},n^2/\sqrt{d}\})$ time. %\rnote{red text is new/edited further than it was earlier.}
%{\color{red} Namely, the lower bounds for DP-SCO and DP-ERM have the forms $\Omega(\frac{1}{\sqrt{n}}+\frac{\sqrt{d}}{n})$ and $\Omega(\frac{\sqrt{d}}{n})$, respectively.} 
% \rnote{The paragraph below in brown color can be taken off since this is not a new result. We already briefly mentioned the lower bound above.} \cnote{Agreed.} {\color{brown} Our proof for the lower bound is based on the fingerprinting code argument due to \cite{bun2018fingerprinting}. In particular, we prove a reduction from DP-ERM in this setting to privately estimating 1-way marginals. Our lower bound does not follow from prior work that used a similar reduction argument, e.g., \cite{BST}, as it requires new tools beyond what is readily available in the $\ell_2$ setting. In particular, our reduction  crucially relies on the strong convexity properties of $\ell_p$ spaces for $1<p\leq 2$ \cite{ball1994sharp}.}

% (\textcolor{red}{ellaborate here... RB: I will shortly}) and makes crucial use of the strong convexity properties of $\ell_p$-spaces \cite{Ball:1994}.
%\rnote{Made substantial edits to following paragraph. Please check.}
Our algorithm is based on a new variant of variance-reduced \emph{Stochastic Frank-Wolfe} (SFW) algorithm with noisy gradient estimates. The noisy gradient estimates are generated using our novel generalized Gaussian mechanism. This algorithm enjoys many attractive features: it is projection free and makes implicit use of gradients through a linear optimization oracle. The structure of our algorithm is significantly different from the non-private variants  \cite{Fang:2018,zhang2020one}, and crucially uses the regularity of the dual space to control the bias of the noisy gradient estimates. In particular, our algorithm involves a binary-tree-based variance reduction technique proposed by Asi et al.~\cite{Asi:2021}. The binary tree is constructed such that more samples (larger batch sizes) are assigned to vertices that are closer to the root and a gradient estimate at each vertex is then calculated using the samples of that vertex and the prior gradient estimates along
the path to the root. The tree-based structure and batch sizing, together with carefully tuned step sizes, effectively control the privacy budget while limiting the down stream influence of the error in the gradient estimates. Moreover, our algorithm makes a single pass on the
input sample, i.e., it runs in linear time. The resulting excess risk of our algorithm is nearly optimal, namely, $\tilde{O}(\frac{1}{\sqrt{n}}+\frac{\sqrt{d}}{\varepsilon n})$. %\cnote{Just make sure we are not including $\varepsilon$ here?} \rnote{Added it.} 
Furthermore, we show that our upper bound is attained with high probability over the input sample and the internal randomness of the algorithm. %\rnote{Added a comment on the high-probability guarantee. See if you want to elaborate further.} \cnote{I think this is OK. We elaborate further in the comparison to the COLT version anyway.}

%\rnote{Made several edits to the following paragraph as well.}

\mypar{Case of $p=1$:} We provide an algorithm that attains nearly-optimal excess population risk with high probability. 
Our algorithm here is based on the variance-reduced one-sample stochastic Frank-Wolfe algorithm \cite{zhang2020one}. This algorithm %enjoys many attractive features: it is projection free and makes implicit use of gradients through a linear optimization oracle; it 
uses a single data point per iteration, allowing for larger number of iterations; and it achieves the optimal excess risk in non-private SCO in the Euclidean setting. 
%The first feature is extremely important for our results, given that the implicit use of gradients is key in order to mitigate dimension-dependence in the private setting.  Regarding the second feature, previous stochastic F-W algorithms required large minibatches \cite{HazanLuo:2016, Hassani:2020}, which limit the number of steps the algorithm can perform, ultimately hurting the convergence rate in the private setting. Finally, regarding the optimality of this method, in this work we adapt the strategy to work in non-Euclidean norms.
Despite its advantages, 
this algorithm does not immediately apply to DP-SCO for the $\ell_1$-setup. The most important reason being that this algorithm was designed for the $\ell_2$-setup, so our first goal is to show that a recursive gradient estimator used in \cite{zhang2020one} (which is a variant of the Stochastic Path-Integrated Differential EstimatoR, SPIDER \cite{Fang:2018}) does indeed work in the $\ell_1$-setup. 
Similarly to the previous case, by making use of the regularity of the $\ell_{\infty}$-norm to obtain
%This requires controlling the variance of a martingale in $\ell_{\infty}$ which is a $O(\ln d)$-regular space. 
%Then, using 
variance bound on the gradient estimator, we are able to extend the SFW method to the $\ell_1$-setup. A second challenge comes from the requirement of differential privacy. First, we use the fact that at each iteration, only a linear optimization oracle is required, and when the feasible set is polyhedral, we can construct such an oracle privately by
the report noisy max mechanism  \cite{DR14,bhaskar2010discovering}. This technique was first used by \cite{TTZ15a} in their construction for the DP-ERM version of this problem which was based on ``ordinary'' full-batch FW. However, the recursive estimator in our construction is queried multiple times where a growing batch is used each time. This requires a different privacy analysis that entails bounding the sensitivity of the gradient estimator by analyzing a recursive bound on the same. The analysis crucially relies on the fact that our construction, unlike the non-private SFW, uses a constant step size and uses a large batch at the first iteration to reduce the sensitivity. % the  In order to certify privacy for the whole trajectory, we carry out a novel privacy analysis for the recursive gradient estimator combined with report noisy max.  Our privacy analysis uses the fact that, unlike the non-private version of SFW, our construction uses a large batch in the first iteration and uses a small, constant step size to reduce the sensitivity of the gradient estimate.

\mypar{Case of $2<p\leq \infty$:} Another interesting question is what happens in the range of $p>2$. For comparison, it is known that non-privately the excess risk behaves as $\Theta(\frac{1}{n^{1/p}}+\frac{d^{1/2-1/p}}{\sqrt n})$. %The first term reflects the high dimensional behavior of the risk, whereas the second one reflects its low dimensional behavior. 
We show that in the low dimensional regime, $d\lesssim n$, the randomized smoothing algorithm of \cite{kulkarni2021private} attains nearly optimal excess risk for general convex losses, and the phased stochastic gradient descent algorithm of \cite{feldman2020private} attains nearly optimal excess risk for smooth convex losses. This implies that for $p=\infty$, these algorithms attain nearly optimal excess risk for general and smooth losses, respectively. Note that the algorithm of \cite{kulkarni2021private} runs in $O(n^{5/4}d^{1/8})$ time, whereas that of \cite{feldman2020private} 
runs in linear time. \rnote{I edited the previous few lines. I don't think we need to talk about expectation vs. high-probability here, but I added a remark on this in the relevant section.}  \cnote{I agree with the changes.} %\rnote{I edited the last sentence.} \cnote{Looks good.}
%\anote{Add note about high probability and also change to iterative localization algorithm}\vspace{0.2cm}
% We show that in the low dimensional regime, $d\lesssim n$, the noisy SGD method  \cite{bassily2014differentially,bassily2020stability} achieves the optimal excess risk. This implies that for $p=\infty$, noisy SGD is optimal.\vspace{0.2cm}

\noindent To conclude our overview, we note that the SFW-based algorithms for the cases $p=1$ and \mbox{$1<p\leq 2$} %\rnote{again, replaces $p<2$ with $p\leq 2$} \cnote{Agreed} 
run in time {\em linear in the dataset size $n$} and are {\em projection-free}, which are desirable properties for large dataset and high-dimensional regimes. %\cnote{Added projection-free here.} 
%which is a desirable property for the large data size regime. 
{Moreover, all of our SFW-style algorithms enjoy high-probability excess risk guarantees. Instrumental to these guarantees is the use of concentration properties of martingale difference sequences in regular spaces \cite{Juditsky:2008}.}

In independent work,\footnote{This work is also concurrent to the conference paper \cite{BGN:2021}.}
%and concurrent work \rnote{This is not a concurrent work to this paper; it was concurrent to the conference version. But I understand it's tricky since we still include results from the conference version.} \cnote{we can just say independent (or say in parenthesis concurrent to the COLT paper \cite{BGN:2021})}, 
Asi et al.~\cite{AFKT:2021} provide sharp upper bounds for DP-SCO in $\ell_1$ setting, for both smooth and nonsmooth objectives. Their algorithm for the smooth case is similar to our polyhedral Stochastic Frank-Wolfe method, where their improvements are obtained by a careful privacy accounting using a binary-tree technique. On the other hand, for the $\ell_p$ setting, when $1<p<2$, they give an algorithm that combines the iterative localization approach of Feldman et al.~\cite{feldman2020private}
and regularized mirror descent. Their algorithm attains nearly optimal excess risk, however, it runs in super-linear ($O(\min\{n^{3/2},n^2/\sqrt{d}\})$) time which is prohibitive in practice. %On the other hand, their work also provides nearly-optimal risk for the $\ell_p$ setting, when $1<p<2$. 
% \rnote{I don't think we need the following brown text. We already give a more significant improvement via a different algorithm. I am fine with keeping it though if we don't need to save on space.} \cnote{Yes. Let's remove it} {\color{brown} Interestingly, their sequential regularization approach can be further refined by using our generalized Gaussian mechanism, removing the additional poly-logarithmic factors in dimension present by their use of the standard Gaussian mechanism. We observe that the optimality of this method is certified by our lower bounds in Section ??.}
To conclude our comparison, we observe that our Generalized Gaussian mechanism allows us to substantially extend the applicability of the 
%Noisy Mirror-Descent \cnote{Keeping it?} and 
Noisy Stochastic Frank-Wolfe method (in Section~\ref{sec:varTreeSFW}) to arbitrary normed spaces with a regular dual.

\mypar{Other Related Work:} Before \cite{TTZ15a}, there have been a few works that studied DP-ERM and DP-SCO in special cases of the $\ell_1$ setting. The works \cite{KST12} and \cite{ST13} studied DP-ERM for $\ell_1$ regression problems; however, they make strong assumptions about the model (e.g., restricted strong convexity).
DP-ERM and DP-SCO for generalized linear models (GLMs) was studied in \cite{JTOpt13}.
%\cite{JTOpt13} studied DP-ERM and DP-SCO in the special case of generalized linear models (GLMs). 
Their bound for DP-ERM was suboptimal, and their generalization error bound relies on special structure of GLMs, where such a bound can be obtained by standard uniform-convergence arguments. We note that such an argument does not lead to optimal bounds for general convex losses.

\mypar{Comparison to our conference paper \cite{BGN:2021}:} %\rnote{``our conference paper'' rather than ``conference version''; and also added citation.} \cnote{Agreed.}
This work is a substantial extension of results published at the Conference on Learning Theory \cite{BGN:2021}. Some of the main innovations upon its conference version include the first linear time algorithm for DP-SCO for $\ell_p$ setup when $1<p<2$, and high-probability guarantees for the variance-reduced SFW algorithms (both for $p=1$ and $1<p\leq 2$). Importantly, we have unified and stregthened our gradient estimator bias bounds to yield high-probability guarantees (see Section~\ref{sec:BiasRecursiveGradient}). This new analysis is also tighter, which leads to better dependence of the resulting excess risk on the regularity parameter, shaving off some polylogarithmic factors in $d$ for $\ell_p$ setups.  
We emphasize as well the generality of these results, as being applicable to arbitrary spaces with a regular dual. In particular, we show that under $\kappa$-regularity of the dual space, DP-SCO with smooth losses is solvable in linear time with excess risk $O\big(\sqrt{\frac{\kappa}{n}}+\frac{\kappa\sqrt{d}}{\varepsilon n}\big)\big)$ 
%\cnote{Is it now $O\big(\sqrt{\frac{\kappa}{n}}+\frac{\kappa\sqrt{d}}{\varepsilon n}\big)\big)$?}\rnote{As in my earlier comment, the current bound/analysis of Thm 5.3 suggests that this is the case (tighter than an earlier version of the analysis). @Anupama: Please, check carefully.}. 
Finally, for $\ell_p$ setups we propose a fast sampling method for the generalized Gaussian mechanism (see Remark~\ref{rem:sampling_GG}).

%% file: prelim.tex
\section{Preliminaries}\label{sec:prelim}

%\paragraph{Notation.~} We denote the data universe by $\cZ$, and $\cD$ is a target distribution supported on $\cZ$. A sample dataset of $n$ i.i.d. draws from $\cD$ is denoted by $S = (z_1,\ldots,z_n)$. {\color{red} I don't think this line is needed; we can directly define this notation the first time it's used. This is taking space and looks sparse.\\ CG: Agreed.}

\paragraph{Normed Spaces and Regularity.~} 
Let $(\bE,\|\cdot\|)$ be a normed space of dimension $d,$ %<\infty$, 
and let $\langle \cdot,\cdot\rangle$ an arbitrary inner product over $\bE$ (not necessarily inducing the norm $\|\cdot\|$). Given $x\in \bE$ and $r>0$, let ${\cal B}_{\|\cdot\|}(x,r)=\{y\in \bE:\|y-x\|\leq r\}$. 
The dual norm over $\bE$ is defined as usual, $\|y\|_{\ast}:=\max_{\|x\|\leq 1} \langle y,x\rangle$. With this definition, $(\bE,\|\cdot\|_{\ast})$ is also a $d$-dimensional normed space. As a main example, consider the case of $\ell_p^d\triangleq(\R^d,\|\cdot\|_p)$, where $1\leq p\leq \infty$ and $\|x\|_p\triangleq\big(\sum_{j\in[d]} |x_j|^p \big)^{1/p}$. As a consequence of the H\"older inequality, one can prove that the dual of $\ell_p^d$ corresponds to $\ell_q^d$, where $1\leq q\leq \infty$ is the conjugate exponent of $p$, determined by $\frac1p+\frac1q=1$.
%We use $\cC$ to denote the parameter space. 
% which is assumed to be a convex polytope. 
%The $l_1$ diameter of $\cC$ is denoted by $\lone{\cC}$. Let $l:\cC \times \cZ \rightarrow \R$ be a convex loss function over $\cC$, which takes a parameter vector $\theta \in \cC$ and a data point $z \in \cZ$ as inputs and outputs a real value. Also $\forall z \in \cZ, ~l(\cdot,z)$ is $L$-Lipschitz with respect to $l_1$ norm. 

%{\color{red}The use of the character ($'$) in $\|\cdot\|'$ is a bit confusing. It is clear that the statements here apply to any norm (primal or dual) so it should be obvious if we just express it as $\|\cdot\|$. I suggest we remove the $'$\\ CG: Done}
The algorithms we consider in this work can be applied to general spaces whose dual has a sufficiently smooth norm. To quantify this property, we use the notion of {\em regular spaces}, following \cite{Juditsky:2008}. Given $\kappa\geq 1$, a normed space $(\bE,\regnorm{\cdot})$ is $\kappa$-regular if there exists $1\leq \kappa_+\leq\kappa$ and a norm $\|\cdot\|_+$ such that $(\bE,\|\cdot\|_{+})$ is $\kappa_+$-smooth, i.e.,
\begin{equation}\label{eqn:k_smooth}
    \|x+y\|_+^2 \leq \|x\|_+^2 +\langle \nabla(\|\cdot\|_+^2)(x),y\rangle+\kappa_+ \|y\|_+^2 \qquad(\forall x,y\in \bE),
\end{equation}
and $\regnorm{\cdot}$ and $\|\cdot\|_{+}$ are equivalent with constant $\sqrt{\kappa/\kappa_+}$:
\begin{equation}\label{eqn:k_equiv_norm}
 \regnorm{x}^2\leq \|x\|_+^2\leq (\kappa/\kappa_+) \regnorm{x}^2 \qquad(\forall x\in\bE).    
\end{equation}

\noindent As basic example, Euclidean spaces are $1$-regular. %\cnote{Got slightly sharper constants for regularity.} %On the other hand, any $d$-dimensional normed space is $d$-regular. This can be obtained by using the John ellipsoid as the unit ball of norm $\|\cdot\|_+$.
Other examples of regular spaces are 
%\begin{itemize}
%\item 
$\ell_q^d$ where $2\leq q\leq \infty$: these spaces are $\kappa$-regular with 
$\kappa=\min\{q-1,2e\log(d)\}$; in this case, $\|x\|_+=\|x\|_r$  %$\big(\sum_{j\in[d]}|x_j|^{r} \big)^{1/r}$ 
with $r=\min\{q,2\log(d)+1\}$ and $\kappa_+=(r-1)$ \cite{Juditsky:2008,vdGeer:2010} (smoothness of this norm squared %function 
is proved e.g.~in \cite[Example 5.11]{Beck:2017}.) %\rnote{did minor rephrasing here.}
%\rnote{It would certainly be better if we explicitly state what is $\kappa+_$ here (like we do below)} \cnote{Done} %\rnote{Perhaps, we should explicitly say here what is $\|\cdot\|_+$ for an $\ell_p$ norm.}
%\item For $2\leq p\leq \infty$, then ${\cal S}_p^d:=(\RR^{d_1\times d_2},\|\cdot\|_{{\cal S}_p})$ is $\kappa$-regular with $\kappa=\min[\max\{2,p-1\},(2\ln(\min\{d_1,d_2\}+2)-1)e]$.
%\end{itemize}
Finally, consider a polyhedral norm $\|\cdot\|$ with unit ball ${\cal B}_{\|\cdot\|}=\mbox{conv}({\cal V})$.
%$\regnorm{x}=\max_{i\in [K]}|\langle v_i,x\rangle|$, where $(v_i)_{i\in[K]}$ generate $\bE$. 
Then $(\bE,\|\cdot\|_{\ast})$ is $(2e\ln |{\cal V}|)$-regular. More precisely, note that $\|x\|_{\ast}=\max_{v\in {\cal V}}|\langle v,x\rangle|$, hence 
the norm $\|x\|_+:=\big(\sum_{v\in{\cal V}}|\langle v,x\rangle|^q \big)^{1/q}$, with $q=2\ln|{\cal V}|+1$, satisfies \eqref{eqn:k_smooth} with $\kappa_+=(q-1)$ (e.g., follows from \cite[Example 5.11]{Beck:2017}), and satisfies \eqref{eqn:k_equiv_norm} with $\sqrt{\kappa/\kappa_+}\leq \sqrt{e}$ (using the equivalence of $\|\cdot\|_q$ and $\|\cdot\|_{\infty}$), thus $\kappa=e\kappa_+=2e\ln|{\cal V}|$. 

%INCLUDE IN ARXIV VERSION
An interesting property of regular spaces are the variance and concentration bounds of their vector-valued martingales \cite{Juditsky:2008}. These will allow us to naturally extend some estimators proposed in Euclidean settings to non-Euclidean regular settings. 

\ifarxiv
\begin{proposition}[Theorem 2.1 in~\cite{Juditsky:2008} restated]\label{prop:jud_mart}
Let $\psi^\infty= \{\psi_i>0\}_{i=1}^{\infty}$ be a sequence of (deterministic) positive reals. Let $(\bE,\regnorm{\cdot})$ be a $\kappa$-regular space, and let $\mg^\infty = \{\mg_i\}_{i=1}^{\infty}$ be a martingale-difference with values in $\bE$, w.r.t.~a filtration $({\cal F}_t)_t$. Suppose that $\forall i\geq 1,~ \mg_i$ satisfy 
$$\ex{}{\exp\{(\regnorm{\mg_i}/\psi_t)^2\}|{\cal F}_{i-1}} \leq \exp\{1\} \quad \text{almost surely.}$$ Let $\bd_T= \sum_{i=1}^{T} \mg_i$, then for all $T\geq 1, \tau \geq 0$:
$$\textstyle \PP\Big[\regnorm{\bd_T}\geq (\sqrt{2e\kappa}+\sqrt{2}\tau)\Big(\sum\limits_{i\leq T}\psi_i^2\Big)^{1/2}\Big]\leq 2\exp\big\{-\tau^2/3\big\}.$$
\end{proposition}

\else
\begin{proposition}[Theorem 2.1 in~\cite{Juditsky:2008} restated]\label{prop:jud_mart}
Let $\psi^\infty= \{\psi_i>0\}_{i=1}^{\infty}$ be a sequence of (deterministic) positive reals. Let $(\bE,\regnorm{\cdot})$ be a $\kappa$-regular space, and let $\mg^\infty = \{\mg_i\}_{i=1}^{\infty}$ be a martingale-difference with values in $\bE$, w.r.t.~a filtration $({\cal F}_t)_t$. Suppose that $\forall i\geq 1,~ \mg_i$ satisfy $\ex{}{\exp\{(\regnorm{\mg_i}/\psi_t)^2\}|{\cal F}_{i-1}} \leq \exp\{1\}$ almost surely. Let $\bd_T= \sum_{i=1}^{T} \mg_i$, then for all $T\geq 1, \tau \geq 0$:
$$\textstyle \PP\Big[\regnorm{\bd_T}\geq (\sqrt{2e\kappa}+\sqrt{2}\tau)\Big(\sum\limits_{i\leq T}\psi_i^2\Big)^{1/2}\Big]\leq 2\exp\big\{-\tau^2/3\big\}.$$
\end{proposition}

\fi

\begin{definition}[Differential Privacy \cite{DKMMN06,DMNS06, DR14}]
Let $\varepsilon,\delta >0$. A (randomized) algorithm $M:\cZ^n \rightarrow \cR$ is $(\varepsilon,\delta)$-differentially private (also denoted $(\varepsilon,\delta)$-DP) %\cnote{(also denoted $(\varepsilon,\delta)$-DP?} 
if for all pairs of datasets $S, S' \in \cZ$ that differ in exactly one entry, and every measurable $\cO \subseteq \cR$, we have:  
$$\Pr \left(M(S) \in \cO \right) \leq e^\varepsilon \cdot \Pr \left(M(S') \in \cO \right) +\delta.$$ 
When $\delta=0$, $M$ is referred to as $\varepsilon$-differentially private.
%When $\delta=0$, it is known as \emph{pure} differential privacy, and parameterized only by $\varepsilon$ in this case. 
\end{definition} 

\ifarxiv
\begin{lem}[Advanced composition \cite{DRV10,DR14}]\label{lem:adv_comp}
For any $\varepsilon > 0, \delta \in [0,1),$ and $\delta' \in (0,1)$, the class of $(\varepsilon,\delta)$-DP algorithms satisfies $(\varepsilon', k\delta + \delta')$-differential privacy under $k$-fold adaptive composition,
for $\varepsilon'= \varepsilon\sqrt{2k \log(1/\delta')} + k\varepsilon(e^\varepsilon-1)$.
\end{lem}

\else

\begin{lemma}[Advanced composition \cite{DRV10,DR14}]\label{lem:adv_comp}
For any $\varepsilon > 0, \delta \in [0,1),$ and $\delta' \in (0,1)$, the class of $(\varepsilon,\delta)$-DP algorithms satisfies $(\varepsilon', k\delta + \delta')$-differential privacy under $k$-fold adaptive composition,
for $\varepsilon'= \varepsilon\sqrt{2k \log(1/\delta')} + k\varepsilon(e^\varepsilon-1)$.
\end{lemma}
\fi

\paragraph{Differentially Private Stochastic Convex Optimization.~}  %Finally, we formally introduce the DP-SCO problem.
%\textcolor{red}{CG: We can omit this (or the paragraph where we introduce SCO).}
%\begin{definition}[DP-SCO] \label{def:DP_SCO}
Let $(\bE,\|\cdot\|)$ be a normed space, and ${\cal X}\subseteq \bE$ a closed convex set of diameter $M>0$. Given $L_0,L_1>0$, denote ${\cal C}(L_0,L_1)$ the class of functions $f:{\cal X}\mapsto \R$ which are convex; $L_0$-Lipschitz, $f(x)-f(y)\leq L_0\|x-y\|$, for all $x,y\in {\cal X}$; and $L_1$-smooth, $\|\nabla f(x)-\nabla f(y)\|_{\ast}\leq L_1\|x-y\|$, for all $x,y\in {\cal X}$. 
%$f:\bE\times \cZ\mapsto \R$ be such that for all $z\in\cZ$, $f(\cdot,z)$ is convex; $L_0$-Lipschitz, $f(x)-f(y)\leq L_0\|x-y\|$ for all $x,y\in {\cal X}$; and $L_1$-smooth  $\|\nabla f(x)-\nabla f(y)\|_{\ast}\leq L_0\|x-y\|$ for all $x,y\in {\cal X}$. 
Given a {\em loss function} $f:{\cal X}\times\cZ\mapsto \R$ s.t.~$f(\cdot,z)\in {\cal C}(L_0,L_1)$ for all $z\in\cZ$, and a distribution ${\cal D}$ over $\cZ$, the SCO problem corresponds to the minimization of the {\em population risk}, $F_{\cal D}(x)\triangleq\E_{z\sim{\cal D}}[f(x,z)]$ over ${\cal X}$. Let, $F_{\cal D}^{\ast}\triangleq\min_{x\in {\cal X}}F_{\cal D}(x)$. 
Given an algorithm ${\cal A}:\cZ^n\mapsto\bE$, define its {\em excess population risk} as %\cnote{How about using the notation  $\Hp[x]$ for a given hypothesis $x\in {\cal X}$. This would prevent the ambiguity of the dependence on $S$. We can later say $\Hp[{\cal A}(S)]$ for the high proba statements.}
\begin{equation}\label{eqn:excess_risk_algo}
    \Hp[{\cal A}]=F_\cD({\cal A}(S)) -  F_\cD^{\ast}.
\end{equation} 
% \begin{equation}\label{eqn:excess_risk_algo}
%     \Risk_{\cal D}[{\cal A}]=\ex{S \sim \cD^n, \cA}{F_\cD({\cal A}(S)) -  F_\cD^{\ast}}.
% \end{equation} 
Our guarantees on the excess population risk will be expressed in terms of
upper bounds on (\ref{eqn:excess_risk_algo}) that hold with high probability over the randomness of both $S$ and the random
coins of the algorithm. Additionally, the expected excess population risk of $\cA$ is defined as $ \Risk_{\cal D}[{\cal A}]=\ex{S \sim \cD^n, \cA}{\Hp[{\cal A}]}.$   %\cnote{Can we replace the last expression by $\ex{S \sim \cD^n, \cA}{\Hp[{\cal A}]}$?}
The DP-SCO problem in the $(\bE,\|\cdot\|)$-setup corresponds to the setting above, where algorithms are constrained to satisfy $(\varepsilon,\delta)$-differential privacy.
%Hence, the minimax excess risk for this problem is given by
%$$ \inf_{{\cal A} \mbox{\em\footnotesize \,\,is }(\varepsilon,\delta)-\mbox{\em\footnotesize DP}} \sup_{f\in {\cal F}(L_0,L_1)}  \Risk_{\cal D}[{\cal A}]. $$
%\textcolor{red}{CG: Notation for class of objectives?}\\
%is the problem of designing an $(\varepsilon,\delta)$-DP algorithm ${\cal A}:\cZ^n\mapsto\bE$ that attempts to minimize the population risk
%$$ (P_{\cal D}) ~~~\min \Big\{ F_{\cal D}(x)=\E_{\bz\sim{\cal D}}[f(x,\bz)] :x\in {\cal X} \Big\}=:F_{\cal D}^{\ast}.$$
%If we additionally assume that for all $z\in\cZ$, $f(\cdot,z)$ has $L_1$-Lipschitz gradient, then we will say this is the {\em smooth} DP-SCO problem. The accuracy of an algorithm here is measured in terms of its {\em excess population risk}, $\Risk_{\cal D}(x)\triangleqF_{\cal D}(x)-F_{\cal D}^{\ast}.$ {\color{red}RB: I think we should say that we will usually assume smoothness except in one construction in Section 5 (the one based on noisy stochastic mirror descent). \\ CG: Yes. Actually, I even prefer the other way around, and just point out that SMD work in the nonsmooth case.}
%\end{definition}
%\cnote{Are we keeping this last part about DP-ERM? I think it's fine to just remove this and informally refer to DP-ERM in the intro.}\rnote{Agreed. Edited it.}
We distinguish the problem above from its empirical counterpart (DP-ERM), where we are interested in the minimization of the \emph{empirical} average of the loss w.r.t. the input dataset. 
%$\min \big\{ F_{S}(x)=\frac{1}{n}\sum_{z\in S}f(x,z)\,:\,x\in {\cal X} \big\}=:F_{S}^{\ast},$
%and accuracy is measured by the  excess empirical risk, $\Risk_{S}[{\cal A}]\triangleq\E_{\cal A}[F_S({\cal A}(S))-F_S^{\ast}]$.

%% file: SCOp12.tex
\section{Generalized Gaussian Distribution and Mechanism}\label{sec:gen-gauss}

One important requirement for the application of DP stochastic first-order methods is designing the proper private mechanism for an iterative method. If we want to achieve privacy by adding noise to gradients, then we need to do it in a way to achieve privacy from {\em gradient sensitivity}, w.r.t.~the dual norm. With this purpose in mind, we design a new noise addition mechanism that leverages the regularity of the {\em dual space} $(\bE,\|\cdot\|_{\ast})$.

\begin{definition}[Generalized Gaussian distribution and mechanism]\label{def:GG_distrib}
Let $(\bE,\|\cdot\|_{\ast})$ be a $d$-dimensional $\kappa$-regular space with smooth norm $\|\cdot\|_+$. We define the generalized Gaussian (GG) distribution ${\cal GG}_{\|\cdot\|_+}(\mu,\sigma^2)$, as the one with density $g(z)=C(\sigma,d)\exp\{-\|z-\mu\|_{+}^2/[2\sigma^2]\}$, where $C(\sigma,d)=\big[\mbox{Area}(\{\|x\|_+=1\})  \frac{(2\sigma^2)^{d/2}}{2}\Gamma(d/2)\big]^{-1}$, and $\mbox{Area}$ is the $\left((d-1)\text{-dim}\right)$ surface measure on $\R^d$.

Given an algorithm ${\cal A}:\cZ^n\mapsto \bE$ with bounded $\|\cdot\|_{\ast}$-sensitivity:
$\sup_{S\simeq S^{\prime}}\|{\cal A}(S)-{\cal A}(S^{\prime})\|_{\ast}\leq s$, we define the {\em generalized Gaussian mechanism} of ${\cal A}$ with variance $\sigma^2$ as
$ {\cal A}_{\cal GG}(S)\sim {\cal GG}_{\|
\cdot\|_+}({\cal A}(S),\sigma^2). $
\end{definition}

\noindent We first list some basic properties.
\begin{proposition}\label{prop:basic_prop_GG}
\begin{enumerate}
\item[(a)] For any $m\geq 1$, if $z\sim{\cal GG}_{\|\cdot\|_+}(0,\sigma^2)$, then $\mathbb{E}[\|z\|_{+}^m]\leq(2\sigma^2)^{m/2}\Gamma\big(\frac{m+d}{2}\big)/\Gamma\big(\frac{d}{2}\big).$ In particular, $\E[\|z\|_{\ast}^2]\leq\E[\|z\|_{+}^2]\leq d\sigma^2$. 
\item[(b)] For any $\nu \geq 0$, if $z\sim{\cal GG}_{\|\cdot\|_+}(0,\sigma^2)$, then $\mathbb{E}[\exp\{ \|z\|_{\ast}^2/\nu^2 \}] \leq \exp\big\{\frac{d\sigma^2}{\nu^2-2\sigma^2}\big\}.$ In particular, if $\nu = \sigma\sqrt{d+2}$, then $\mathbb{E}[\exp\{ \frac{\|z\|_{\ast}^2}{\nu^2} \}] \leq \E[\exp\{ \frac{\|z\|_{+}^2}{\nu^2} \}] \leq \exp(1).$
%$\mathbb{E}[\exp\{ \|z\|_{\ast}^2/\nu^2 \}] \leq \E[\exp\{ \|z\|_{+}^2/\nu^2 \}]  ,\leq \exp\big\{\frac{d\sigma^2}{\nu^2-2\sigma^2}\big\}.$ 
%In particular, when $\nu=\sigma\sqrt{d+2}$, $\mathbb{E}[\exp\{ \|z\|_{\ast}^2/\nu^2 \}] \leq \E[\exp\{ \|z\|_{+}^2/\nu^2 \}]  \leq \exp(1)$. 
\item[(c)] For any $\alpha>1$, $D_{\alpha}({\cal GG}_{\|\cdot\|_+}(\mu_1,\sigma^2)||{\cal GG}_{\|\cdot\|_+}(\mu_2,\sigma^2))\leq \frac{\kappa\alpha^2}{2\sigma^2(\alpha-1)}\|\mu_1-\mu_2\|_{\ast}^2$, where $D_{\alpha}$ is the $\alpha$-R\'enyi divergence.
\end{enumerate}
\end{proposition}
The last property exhibits a behavior analogous to the Gaussian r.v.s in the $\ell_2$-setup.
% Regarding the moment estimates in part (a) above, notice they imply $\E[\|\bz\|_{\ast}^2]\leq d\sigma^2$. This behavior is analog to the Gaussian mechanism.

%{\color{red}RB: I moved the statement concerning the second moment inside the proposition (since it's important) after making minor edits.}
\ifarxiv
\begin{proof}
\begin{enumerate}
    \item[(a)] Notice that for any $m$ (below $\Gamma(\cdot)$ is the Gamma function),
$$ 
\int_0^{+\infty} \exp\big\{-\frac{r^2}{2\sigma^2}\big\}r^m dr=\frac{(2\sigma^2)^{m+1/2}}{2}\int_0^{+\infty}e^{-u}u^{m-1/2}du = \frac{(2\sigma^2)^{(m+1)/2}}{2}\Gamma\big(\frac{m+1}{2}\big).
$$
This implies that the $m$-th moment w.r.t.~$\|\cdot\|_+$ can be computed as follows
%the normalizing constant can be computed as
%\begin{eqnarray*}
%C(\sigma,d)^{-1} &=&\int_{\R^d}\exp\big\{-\frac{\|z\|_{+}^2}{2\sigma^2}\big\}dz
%=\mbox{Area}(\{\|x\|_+=1\})\int_0^{\infty}\exp\big\{-\frac{r^2}{2\sigma^2}\big\}r^{d-1} dr \\
%&=&\mbox{Area}(\{\|x\|_+=1\})  \frac{(2\sigma^2)^{d/2}}{2}\Gamma(d/2).
%\end{eqnarray*}
\begin{eqnarray*}
\E[\|z\|_{+}^m] &=&C(
\sigma,d)\int_{\R^d}\|z\|_{+}^m\exp\Big\{-\frac{\|z\|_{+}^2}{2\sigma^2}\Big\}dz \\
&=&C(\sigma,d)\mbox{Area}(\{\|x\|_+=1\})\int_0^{\infty}r^{m+d-1}\exp\big\{-\frac{r^2}{2\sigma^2}\big\} dr \\
%&=&C(\sigma,d)\mbox{Area}(\{\|x\|_+=1\}) \frac{(2\sigma^2)^{(m+d)/2}}{2}\Gamma\big(\frac{m+d}{2}\big)\\
%\frac{(2\sigma^2)^{d/2}}{2}\Gamma(d/2).
&=& (2\sigma^2)^{m/2}\Gamma\big(\frac{m+d}{2}\big)/\Gamma\big(\frac{d}{2}\big).
\end{eqnarray*}
We conclude using that $\|z\|_{\ast}\leq\|z\|_{+}$. On the other hand, the bounds for the second moment are obtained by using that $\Gamma(1+d/2)=(d/2)\Gamma(d/2)$.
\item[(b)] Using $\|z\|_{\ast}\leq\|z\|_{+}$, we have
\begin{align*}
	\textstyle\E\big[\exp\big\{ \frac{\|z\|_{\ast}^2}{\nu^2} \big\}\big] 
	&\textstyle \leq \E\big[\exp\big\{\frac{\|z\|_{+}^2}{\nu^2} \big\}\big] 
	=C(\sigma,d){\displaystyle\int_{\R^d}} \exp\big\{\frac{\|z\|_+^2}{\nu^2}\big\}\exp\big\{-\frac{\|z\|_+^2}{2\sigma^2}\big\} dz\\
	&\textstyle=C(\sigma,d){\displaystyle\int_{\R^d}} \exp\big\{\|z\|_+^2\big(\frac{1}{\nu^2}-\frac{1}{2\sigma^2}\big)\big\} dz 
	= C(\sigma,d)/C\big(\big(\frac{1}{\sigma^2}-\frac{2}{\nu^2}\big)^{-1/2},d\big) \\
	&= \Big(\frac{\nu^2}{\nu^2-2\sigma^2} \Big)^{d/2} 
	 = \Big(1+\frac{2\sigma^2}{\nu^2-2\sigma^2} \Big)^{d/2}\\
	&\leq \exp\big\{\frac{d}{2}\frac{2\sigma^2}{\nu^2-2\sigma^2}\big\}.
\end{align*}
Note that, this is upper bounded by $\exp(1)$ when $\nu=\sigma\sqrt{d+2}$.
\item[(c)] Let $\mathbb{P}={\cal GG}(\mu_1,\sigma^2)$ and $\mathbb{Q}={\cal GG}(\mu_2,\sigma^2)$. 
    %Now, we consider the estimation of the R\'enyi divergence between two generalized Gaussian distributions $\mathbb{P}$, $\mathbb{Q}$ with different mean $\mu_1$, $\mu_2$, respectively.
\begin{eqnarray*}
\exp\{(\alpha-1)D_{\alpha}(\mathbb{P}||\mathbb{Q})\} &=& C(\sigma,d) \int_{\R^d} \Big(\frac{d\mathbb{P}}{d\mathbb{Q}}\Big)^{\alpha}d\mathbb{Q} \\
&=& C(\sigma,d)  \int_{\R^d} \exp\Big\{-\frac{\alpha}{2\sigma^2}\|z-\mu_1\|_+^2+\frac{\alpha-1}{2\sigma^2}\|z-\mu_2\|_+^2 \Big\} dz \\
&=& C(\sigma,d) \int_{\R^d} \exp\Big\{-\frac{\alpha}{2\sigma^2}\|z-\mu_1+\mu_2\|_+^2+\frac{\alpha-1}{2\sigma^2}\|z\|_+^2 \Big\} dz.
\end{eqnarray*}
Let now $\mu=\mu_1-\mu_2$ and  $p(\cdot)=\|\cdot\|_+^2$. Now, by convexity and smoothness of $\|\cdot\|_+^2$
$$ -\alpha\|z-\mu\|_+^2 \leq -\alpha\|z\|_+^2+\langle \nabla p(z),\alpha\mu\rangle \leq -\alpha\|z\|_+^2+[\|z\|_+^2-\|z-\alpha\mu\|_+^2+\kappa_+\|\alpha\mu\|_{+}^2]. $$
Plugging this in the integral above, we get
\begin{eqnarray*}
e^{(\alpha-1)D_{\alpha}(\mathbb{P}||\mathbb{Q})} &\leq&\textstyle \exp\big\{\frac{\kappa_+\alpha^2}{2\sigma^2}\|\mu\|_{+}^2\big\}C(\sigma,d) {\displaystyle\int_{\R^d}} \exp\Big\{-\frac{\|z-\alpha\mu\|_+^2}{2\sigma^2}\Big\} dz
\leq \exp\big\{\frac{\kappa\alpha^2}{2\sigma^2}\|\mu\|_{\ast}^2\big\},
\end{eqnarray*}
Hence, $D_{\alpha}(\mathbb{P}||\mathbb{Q}) \leq \frac{\kappa \alpha^2}{2\sigma^2(\alpha-1)}\|\mu\|_{\ast}^2.$
\end{enumerate}
\end{proof}

\else
\begin{proof}
%\begin{enumerate}
For (a) and (b) we refer to Appendix~\ref{sec:apndx_proof_gg}. For (c), let $\mathbb{P}={\cal GG}(\mu_1,\sigma^2)$ and $\mathbb{Q}={\cal GG}(\mu_2,\sigma^2)$. 
    %Now, we consider the estimation of the R\'enyi divergence between two generalized Gaussian distributions $\mathbb{P}$, $\mathbb{Q}$ with different mean $\mu_1$, $\mu_2$, respectively.
\begin{eqnarray*}
\exp\{(\alpha-1)D_{\alpha}(\mathbb{P}||\mathbb{Q})\} &=& C(\sigma,d) \int_{\R^d} \Big(\frac{d\mathbb{P}}{d\mathbb{Q}}\Big)^{\alpha}d\mathbb{Q} \\
&=& C(\sigma,d)  \int_{\R^d} \exp\Big\{-\frac{\alpha}{2\sigma^2}\|z-\mu_1\|_+^2+\frac{\alpha-1}{2\sigma^2}\|z-\mu_2\|_+^2 \Big\} dz \\
&=& C(\sigma,d) \int_{\R^d} \exp\Big\{-\frac{\alpha}{2\sigma^2}\|z-\mu_1+\mu_2\|_+^2+\frac{\alpha-1}{2\sigma^2}\|z\|_+^2 \Big\} dz.
\end{eqnarray*}
Let now $\mu=\mu_1-\mu_2$ and  $p(\cdot)=\|\cdot\|_+^2$. Now, by convexity and smoothness of $\|\cdot\|_+^2$
$$ -\alpha\|z-\mu\|_+^2 \leq -\alpha\|z\|_+^2+\langle \nabla p(z),\alpha\mu\rangle \leq -\alpha\|z\|_+^2+[\|z\|_+^2-\|z-\alpha\mu\|_+^2+\kappa_+\|\alpha\mu\|_{+}^2]. $$
Plugging this in the integral above, we get
\begin{eqnarray*}
e^{(\alpha-1)D_{\alpha}(\mathbb{P}||\mathbb{Q})} &\leq&\textstyle \exp\big\{\frac{\kappa_+\alpha^2}{2\sigma^2}\|\mu\|_{+}^2\big\}C(\sigma,d) {\displaystyle\int_{\R^d}} \exp\Big\{-\frac{\|z-\alpha\mu\|_+^2}{2\sigma^2}\Big\} dz
\leq \exp\big\{\frac{\kappa\alpha^2}{2\sigma^2}\|\mu\|_{\ast}^2\big\},
\end{eqnarray*}
Hence, $D_{\alpha}(\mathbb{P}||\mathbb{Q}) \leq \frac{\kappa \alpha^2}{2\sigma^2(\alpha-1)}\|\mu\|_{\ast}^2.$
%\end{enumerate}
\end{proof}
\fi

\noindent We now provide a short summary of consequences of the GG mechanism. 
%The missing part of Proposition \ref{prop:basic_prop_GG} is a consequence of the following result. 
\begin{corollary}\label{cor:GG_mech_priv}
The generalized Gaussian mechanism applied to a function with $\|\cdot\|_{\ast}$-sensitivity bounded by $s>0$ is $(\alpha,\rho)$-RDP, where $\rho=\kappa\alpha^2s^2/[2\sigma^2(\alpha-1)]$. In particular, choosing $\sigma^2=2\kappa\log(1/\delta)s^2/\varepsilon^2$, the GG mechanism is $(\varepsilon,\delta)$-DP. 
\end{corollary}
\begin{proof}
The first part follows directly from Proposition~\ref{prop:basic_prop_GG}. The second part can be obtained from the first part, together with the DP/RDP reduction in \cite[Propos.~3]{Mironov:2017}.
\end{proof}

\begin{remark} \label{rem:sampling_GG}
Of particular interest in this work is the space $\ell_q^d$ where $2\leq q<\infty$. These spaces are $\kappa$-regular with $\kappa=\min\{q-1,2e\log d\}$, and the corresponding $\kappa_+$-smooth norm is $\|x\|_r$, where $r=\min\{q,2\log(d)+1\}$ thus
$\kappa_+=r-1$. In this case, implementing the GG mechanism can be done in linear (in the dimension) time, assuming access to a sampler from the Gamma distribution. For this, we use the decomposition $X=RU$ for a random vector $X$
with independent coordinates with density proportional to $e^{-t^p}$ \cite{Sinz:2009}. This polar-type decomposition corresponds to $U=X/\|X\|_p$ (uniformly distributed on the $\|\cdot\|_p$ sphere) and $R=\|X\|_p$ (with cdf $\Gamma(d/p,r^p/[2\sigma]^2)/\Gamma(d/p)$), which are independent. On the other hand, it is easy to see that the GG distribution has a similar decomposition $Z=RU$, where $U=Z/\|Z\|_p$ is uniform on the $\|\cdot\|_p$ sphere, but $R$ follows a centered chi distribution. Hence, sampling from $X$ as above, and transforming its radial component from the cdf of $\|X\|_p$ to a centered chi distribution (see discussion in \cite{Sinz:2009}),  provides an exact GG sampler.
\end{remark}

\section{High Probability Bias of Recursive Gradient Estimators} \label{sec:BiasRecursiveGradient}
 %Below we restate Theorem 2.1 in \cite{Juditsky:2008} which we will use to control the variance of a martingale-difference in the $\ell_p$ setups when $1\leq p \leq 2$.

% \anote{changing $\sigma$ in below theorem to $\psi$ to differentiate from GG noise}
% \begin{proposition}[Theorem 2.1~\cite{Juditsky:2008}]\label{prop:jud_mart}
% Let $\psi^\infty= \{\psi_i>0\}_{i=1}^{\infty}$ be a sequence of (deterministic) positive reals. Let $(\bE,\regnorm{\cdot})$ be a $\kappa$-regular space, and let $\mg^\infty = \{\mg_i\}_{i=1}^{\infty}$ be a martingale-difference with values in $\bE$, w.r.t.~a filtration $({\cal F}_t)_t$. Suppose that $\forall i\geq 1,~ \mg_i$ satisfy $\ex{}{\exp\{(\regnorm{\mg_t}/\psi_t)^2\}|{\cal F}_{i-1}} \leq \exp\{1\}$. Let $\bd_T= \sum\limits_{i=1}^{T} \mg_i$, then
% $$\textstyle \PP\Big[\regnorm{\bd_T}\geq (\sqrt{2e\kappa}+\sqrt{2}\tau)\Big(\sum\limits_{i\leq T}\psi_i^2\Big)^{1/2}\Big]\leq 2\exp\big\{-\tau^2/3\big\}.$$
% \end{proposition}
%\textcolor{red}{}
%\noindent We note that will not make direct use of this result, and derive these bounds where necessary, in order to make the arguments self-contained.

%\noindent Note that we will use this result to derive the following high probability bounds.

%\paragraph{High Probability Bias of Recursive Gradient Estimators}
Our DP-SCO algorithms in the $\ell_p$ setups when $1\leq p \leq 2$, are based on the variance-reduced stochastic Frank-Wolfe algorithm variants. These algorithms use a variant of the following recursive gradient estimator: %\cnote{Shouldn't this be just $\alpha$ and not $\alpha_t$? O.w. the lemma doesn't make sense.}\anote{Sorry that was a typo corrected it.}
\begin{align}
    \tnab_t = (1-\alpha)(\tnab_{t-1}+\Delta_t)+\alpha \nabla_t+\bg_t \label{eqn:rec_grad}
\end{align}
For any $t$, let $\bB_t$ denote a mini-batch of data points drawn from the input dataset without replacement. Note that, since the input dataset consists of i.i.d.~draws from an unknown distribution $\cD$, the mini-batches $\{\bB_0,\ldots,\bB_t\}$ are disjoint and independent. Here, given a mini-batch $\bB_t$, $\nabla_t$ is the unbiased gradient estimator given by $\nabla_t=\frac{1}{|\bB_t|}\sum_{z\in\bB_t} \nabla f(x^t,z)$, and  $\Delta_t$ is the gradient variation given by %\cnote{Shouldn't this be with non-capitalized $f$?} \anote{Yes, changed to small $f$.} 
$\Delta_t=\frac{1}{|\bB_t|}\sum_{z\in\bB_t}[\nabla f(x^t,z)-\nabla f(x^{t-1},z)]$.  Here $\alpha \in (0,1)$ is an averaging factor. % and for the sake of simplicity, we assume for all $t$, $\alpha_t =\alpha$. 
Finally, $\bg_t$ is the noise, which is a centered random variable,  $\ex{}{\bg_t}=0$, and satisfies a light-tail property %(for some $\nu_t \geq 0)$,
\begin{align}
    \ex{}{\exp\{\|\bg_t\|^2/\nu_t^2\}} \leq \exp(1) \quad(\exists\nu_t \geq 0). \label{eqn:noise_light_tail}
\end{align}
Recall that $\cX$ is a closed convex set of diameter $M>0$, and for any datapoint $z$ in the input dataset, we have the following properties:
\ifarxiv
\else
\vspace{0.2cm}
\fi
\begin{itemize}
\item Unbiasedness: $\ex{z\in {\cal D}}{\nabla f(x,z)}=\nabla F_{\cal D}(x)$
\item Boundedness: $\|\nabla f(x,z)\|_{\ast}\leq L_0$, a.s..
\item Smoothness: $\|\nabla f(x,z)-\nabla f(y,z)\|_{\ast}\leq L_1\|x-y\|$.
\end{itemize}
\ifarxiv
\else
\vspace{0.2cm}
\fi
We will use Proposition~\ref{prop:jud_mart} and the above properties to derive high probability bounds on the variance of the recursive gradient estimator $\tnab_t$, which is given by the following lemma. 
%Below we state a lemma that bounds the variance of the gradient estimator with high probability. 
\ifarxiv
\begin{lem}\label{lem:hp_var_grad_est_gen}
Let $\alpha \in (0,1)$ and step sizes, $\eta_t \in (0,1), ~\forall t$. Let $(\bE,\|\cdot\|)$ be a normed space such that $(\bE,\|\cdot\|_{\ast})$ is $\kappa$-regular. Suppose the noise $\bg_t$, is a centered random variable, $\ex{}{\bg_t}=0$, and satisfies (\ref{eqn:noise_light_tail}) for some $\nu_t\geq 0$. For any $0<\beta<1$, with probability at least $1-\beta$, the recursive gradient estimate $\tnab_t$ satisfies for all $t\in[T]$:
$$\textstyle\|\tnab_t-\nabla F_{\cal D}(x^t)\|_{\ast} \leq C_{\beta} \big[\frac{(1-\alpha)^t L_0}{\sqrt{|\bB_0|}} +  (L_1M\eta_t+\alpha L_0) \sum\limits_{s<t}\frac{(1-\alpha)^{t-(s+1)}}{\sqrt{|\bB_s|}} +\sum\limits_{s<t}(1-\alpha)^{t-(s+1)}\nu_{s} \big],$$
where $C_{\beta}=\big(\sqrt{e\kappa}+\sqrt{3\log(2T/\beta)}\big)$.
%Then, for all $t$, the recursive gradient estimate $\tnab_t$ satisfies: 
%\cnote{Should it be $\beta/2$, instead of $2/\beta$ in the proba below? And also perhaps it's better to say: ``For any $0<\beta<1$, with probability $\beta/2$, the recursive gradient estimate $\tnab_t$ satisfies for all $t\in[T]$:
%$$\textstyle\|\tnab_t-\nabla F_{\cal D}(x^t)\|_{\ast} \leq C_{\beta} \big[\frac{(1-\alpha)^t L_0}{\sqrt{|\bB_0|}} +  (L_1M\eta_t+\alpha L_0) \sum\limits_{s<t}\frac{(1-\alpha)^{t-(s+1)}}{\sqrt{|\bB_s|}} +\sum\limits_{s<t}(1-\alpha)^{t-(s+1)}\nu_{s} \big],$$
%where $C_{\beta}=\big(\sqrt{e\kappa}+\sqrt{3\log(1/\beta)}\big)$.''}\anote{Yes it should be $2\beta$ instead of $2/\beta$. Corrected and changed the statement.}
% $$\textstyle\pr{}{\|\tnab_t-\nabla F_{\cal D}(x^t)\|_{\ast} \geq  \left(\sqrt{e\kappa}+\sqrt{3\log(\frac{1}{\beta})}\right)\left(\frac{(1-\alpha)^t L_0}{\sqrt{|\bB_0|}} +  (L_1M\eta_t+\alpha L_0) \sum\limits_{s<t}\frac{C}{\sqrt{|\bB_s|}} +\sum\limits_{s<t}C\nu_{s} \right)} \leq \frac{2}{\beta},$$
% where $C= (1-\alpha)^{t-(s+1)}$.
% $$\mathbb{P}\Big[\|\tnab_t-\nabla F_{\cal D}(x^t)\|_{\ast} \geq \sqrt{2}[\sqrt{e\kappa}+\tau]\big(|\bB_0|\sigma_0^2+\sum_{s<t}|\bB_s|\sigma_s^2+\sum_{s<t}\prod_{r=s+1}^t(1-\alpha_r)^2\nu_s^2 \big) \Big] \leq 2\exp\Big\{ -\tau^2/3 \Big\}.$$
\end{lem}
\begin{proof}
Note that, the recursive gradient estimator is given by \ref{eqn:rec_grad}. %given by:
% $$\tnab_t = (1-\alpha)(\tnab_{t-1}+\Delta_t)+\alpha \nabla_t+\bg_t,$$
% where,  $\nabla_t=\frac{1}{|\bB_t|}\sum_{z\in\bB_t} \nabla f(x^t,z)$ is a unbiased gradient estimator for a mini-batch $B_t$, $\Delta_t=\frac{1}{|\bB_t|}\sum_{z\in\bB_t}[\nabla F(x^t,z)-\nabla F(x^{t-1},z)]$ is the (mini-batch) gradient variation, and $\alpha$ is the averaging factor.
%Under these assumptions, we will prove a high probability bound on $\|\tnab_t-\nabla F_{\cal D}(x^t)\|_{\ast}$. 
Now, to prove a high probability bound on $\|\tnab_t-\nabla F_{\cal D}(x^t)\|_{\ast}$, we unravel this bias in terms of a martingale difference sequence. Below, for a compact notation, let $\bDel_t \triangleq \nabla F_{\cD}(x^t)- \nabla F_{\cD}(x^{t-1}). $
%\cnote{Used small here to make everything fit better. If you don't like it, you can use the original (which I left commented out)}
\ifarxiv
\begin{align*}
 &\textstyle   \tnab_t - \nabla F_\cD(x^t)\\ =&(1-\alpha)\big[\tnab_{t-1} - \nabla F_{\cD}(x^{t-1})\big] + (1-\alpha) \nabla F_{\cD}(x^{t-1}) - \nabla F_\cD(x^t)  + (1-\alpha)\Delta_t  + \alpha \nabla_t + \alpha \bg_t\\
 =&\textstyle(1-\alpha)\big[\tnab_{t-1} - \nabla F_{\cD}(x^{t-1})\big] +  
    \left[\Delta_t - \bDel_t \right] + \alpha\left[\nabla_t - \nabla F_{\cD}(x^t)     +  \bg_t \right]\\
= &   (1-\alpha)^t (\tnab_0-\nabla F_{\cal D}(x^0)) +\sum\limits_{j=1}^t (1-\alpha)^{t-(j+1)}[(\Delta_j-\bDel_j) +\alpha(\nabla_{j-1}-\nabla F(x^{j-1}) + \bg_j)].
\end{align*}

\else

\begin{small}
\begin{align*}
 &\textstyle   \tnab_t - \nabla F_\cD(x^t)\\ =&(1-\alpha)\big[\tnab_{t-1} - \nabla F_{\cD}(x^{t-1})\big] + (1-\alpha) \nabla F_{\cD}(x^{t-1}) - \nabla F_\cD(x^t)  + (1-\alpha)\Delta_t  + \alpha \nabla_t + \alpha \bg_t\\
 =&\textstyle(1-\alpha)\big[\tnab_{t-1} - \nabla F_{\cD}(x^{t-1})\big] +  
    \left[\Delta_t - \bDel_t \right] + \alpha\left[\nabla_t - \nabla F_{\cD}(x^t)     +  \bg_t \right]\\
= &   (1-\alpha)^t (\tnab_0-\nabla F_{\cal D}(x^0)) +\sum\limits_{j=1}^t (1-\alpha)^{t-(j+1)}[(\Delta_j-\bDel_j) +\alpha(\nabla_{j-1}-\nabla F(x^{j-1}) + \bg_j)].
\end{align*}
\end{small}
\fi

Given $t\in[T]$ and $z\in\bB_t$, let $\cF_{<z,t}$ be the $\sigma$-algebra induced by all datapoints preceding $z$ (in the order used by the algorithm), as well as $(\bg_s)_{s<t}$.
Here, we assume that $\tnab_t - \nabla F_\cD(x^t)$ is a sum of martingale-difference terms and hence, we need to establish the following bounds for each summand in the estimators.

%To obtain a sharp bound, we also need to explicitly establish bounds for each summand in the estimators. Given $t\in[T]$ and $z\in\bB_t$, let  ${\cal F}_{<z,t}$ is the sigma algebra induced by all datapoints preceding $z$ (in the order used by the algorithm), as well as $(\bg_s)_{s<t}$.
%\vspace{0.2cm}
\begin{enumerate}[leftmargin=*]
\item First term: $\tnab_0-\nabla F_{\cal D}(x^0)$.\\
For the initial batch $\bB_0$, we have that 
\begin{align*}
\tnab_0-\nabla F_{\cal D}(x^0)=\frac{1}{|\bB_0|}&\sum_{z\in \bB_0}[\nabla f(x^0,z)-\nabla F_{\cal D}(x^0)]
\end{align*}
Moreover, $\|\nabla f(x^0,z)-\nabla F_{\cal D}(x^0)\|_{\ast} \leq {2L_0}$, for each $z \in \bB_0$, a.s.

In particular, each summand of the first term satisfies, for $\psi_0^2=4L_0^2(1-\alpha)^{2t}/\bB_0^2$,
\[ \ex{}{\exp\Big\{(1-\alpha)^{2t}\|\nabla f(x^0,z)-\nabla F_{\cal D}(x^0)\|^2_{\ast}/[\psi_0\bB_0]^2 \Big\} \Big| {\cal F}_{<z,0} } \leq \exp(1).\]
\item Second term: $[\Delta_j-\bDel_j]+\alpha[\nabla_{j-1}-\nabla F(x^{j-1})]$.\\
Given a mini-batch $\bB_j$, we can write this term as follows
\begin{small}
\begin{align*}
\textstyle [\Delta_j-\bDel_j]+\alpha[\nabla_{j-1}-\nabla F(x^{j-1})] &= \frac{1}{|\bB_j|}\sum_{z\in\bB_j}\Big[(\nabla f(x^j,z)-\nabla f(x^{j-1},z)-\bDel_j)\\
&\textstyle \qquad + \alpha(\nabla f(x^{j-1},z)-\nabla F(x^{j-1})) \Big].    
\end{align*}
\end{small}
Now, using that $f(\cdot,z)$ is $L_0$-Lipschitz and $L_1$-smooth, we have
\begin{align*}
\|\nabla f(x^j,z)-\nabla f(x^{j-1},z)-\bDel_j\|_{\ast} \leq 2L_1M\eta_j \\ \|\nabla f(x^{j-1},z)-\nabla F(x^{j-1}) \|_{\ast} \leq 2L_0.
\end{align*}
Thus, the dual norm of each summand is a.s.~bounded by $2L_1M\eta_j+2\alpha L_0$. Finally, choosing $$\psi_j^2=\frac{4}{|\bB_j|^2}(1-\alpha)^{2t-2(j+1)}[L_1M\eta_j+\alpha L_0]^2,$$
\noindent for each summand of the second term, using datapoint $z\in \bB_j$, we have
\begin{align*}
\mathbb{E}\Big[ \exp\Big\{ \frac{1}{(\psi_t|\bB_j|)^2} &(1-\alpha)^{2t-2(j+1)}
\Big\|(\nabla f(x^j,z)-\nabla f(x^{j-1},z)-\bDel_j)\\
&+\alpha[\nabla f(x^{j-1},z)-\nabla F(x^{j-1})\Big\|_{\ast}^2  \Big\} \Big| {\cal F}_{<z,j} \Big] \leq \exp(1),
\end{align*}
\item Third term: $\bg_j$\\
It is given that $\bg_j$ is a centered random variable, $\ex{}{\bg_t}=0$, and it satisfies %the light-tail property (for some $\nu_t\geq 0$)
$$\ex{}{\exp\{\|\bg_t\|^2/\nu_t^2\}} \leq \exp(1). $$
\end{enumerate}
%\vspace{0.2cm}
\noindent Therefore, by Proposition~\ref{prop:jud_mart}, we have that for every $t$ and any $\beta \in (0,1)$:
\[
\mathbb{P}\Big[\|\tnab_t-\nabla F_{\cal D}(x^t)\|_{\ast} \geq \sqrt{2}C_\beta\Big(|\bB_0|\psi_0^2+\sum_{s<t}(1-\alpha)^{2t-2(s+1)}\big[|\bB_s|\psi_s^2+\nu_s^2 \big]\Big)^{1/2} \Big] \leq \beta,
\]
where $C'_\beta =(\sqrt{e\kappa}+\sqrt{3\log(2/\beta)}).$ Using Jensen’s inequality and union bound over $t$, with probability $1-\beta$, the recursive gradient estimate $\tnab_t$ satisfies for all $t\in[T]$ (letting $C_\beta =(\sqrt{e\kappa}+\sqrt{3\log(2T/\beta)})$): %\cnote{The JN statetment does not have a union bound incorporated. So we need to correct this here, using a union bound over $t$. It should only affect $C_\beta =(\sqrt{e\kappa}+\sqrt{3\log(2T/\beta)}).$}\anote{Corrected} \cnote{Agreed. Only the inequality was going in the wrong direction}
$$\textstyle\|\tnab_t-\nabla F_{\cal D}(x^t)\|_{\ast} \leq C_{\beta} \big[\frac{(1-\alpha)^t L_0}{\sqrt{|\bB_0|}} +  (L_1M\eta_t+\alpha L_0) \sum\limits_{s<t}\frac{(1-\alpha)^{t-(s+1)}}{\sqrt{|\bB_s|}} +\sum\limits_{s<t}(1-\alpha)^{t-(s+1)}\nu_{s} \big].$$
\end{proof}

\else

Due to space considerations, we defer the proof of the lemma to the appendix.

\begin{lemma}\label{lem:hp_var_grad_est_gen}
Let $\alpha \in (0,1)$ and step sizes, $\eta_t \in (0,1), ~\forall t$. Let $(\bE,\|\cdot\|)$ be a normed space such that $(\bE,\|\cdot\|_{\ast})$ is $\kappa$-regular. Suppose the noise $\bg_t$, is a centered random variable, $\ex{}{\bg_t}=0$, and satisfies (\ref{eqn:noise_light_tail}) for some $\nu_t\geq 0$. For any $0<\beta<1$, with probability at least $1-\beta$, the recursive gradient estimate $\tnab_t$ satisfies for all $t\in[T]$:
$$\textstyle\|\tnab_t-\nabla F_{\cal D}(x^t)\|_{\ast} \leq C_{\beta} \big[\frac{(1-\alpha)^t L_0}{\sqrt{|\bB_0|}} +  (L_1M\eta_t+\alpha L_0) \sum\limits_{s<t}\frac{(1-\alpha)^{t-(s+1)}}{\sqrt{|\bB_s|}} +\sum\limits_{s<t}(1-\alpha)^{t-(s+1)}\nu_{s} \big],$$
where $C_{\beta}=\big(\sqrt{e\kappa}+\sqrt{3\log(2T/\beta)}\big)$.
%Then, for all $t$, the recursive gradient estimate $\tnab_t$ satisfies: 
%\cnote{Should it be $\beta/2$, instead of $2/\beta$ in the proba below? And also perhaps it's better to say: ``For any $0<\beta<1$, with probability $\beta/2$, the recursive gradient estimate $\tnab_t$ satisfies for all $t\in[T]$:
%$$\textstyle\|\tnab_t-\nabla F_{\cal D}(x^t)\|_{\ast} \leq C_{\beta} \big[\frac{(1-\alpha)^t L_0}{\sqrt{|\bB_0|}} +  (L_1M\eta_t+\alpha L_0) \sum\limits_{s<t}\frac{(1-\alpha)^{t-(s+1)}}{\sqrt{|\bB_s|}} +\sum\limits_{s<t}(1-\alpha)^{t-(s+1)}\nu_{s} \big],$$
%where $C_{\beta}=\big(\sqrt{e\kappa}+\sqrt{3\log(1/\beta)}\big)$.''}\anote{Yes it should be $2\beta$ instead of $2/\beta$. Corrected and changed the statement.}
% $$\textstyle\pr{}{\|\tnab_t-\nabla F_{\cal D}(x^t)\|_{\ast} \geq  \left(\sqrt{e\kappa}+\sqrt{3\log(\frac{1}{\beta})}\right)\left(\frac{(1-\alpha)^t L_0}{\sqrt{|\bB_0|}} +  (L_1M\eta_t+\alpha L_0) \sum\limits_{s<t}\frac{C}{\sqrt{|\bB_s|}} +\sum\limits_{s<t}C\nu_{s} \right)} \leq \frac{2}{\beta},$$
% where $C= (1-\alpha)^{t-(s+1)}$.
% $$\mathbb{P}\Big[\|\tnab_t-\nabla F_{\cal D}(x^t)\|_{\ast} \geq \sqrt{2}[\sqrt{e\kappa}+\tau]\big(|\bB_0|\sigma_0^2+\sum_{s<t}|\bB_s|\sigma_s^2+\sum_{s<t}\prod_{r=s+1}^t(1-\alpha_r)^2\nu_s^2 \big) \Big] \leq 2\exp\Big\{ -\tau^2/3 \Big\}.$$
\end{lemma}

\fi 

\section{Differentially Private SCO: $\ell_p$-setup for $1<p\leq 2$}

\subsection{Noisy Variance-Reduced Stochastic Frank-Wolfe}\label{sec:varTreeSFW}
In this section, we study DP-SCO in the $\ell_p$-setup when $1<p\leq2$ and provide a high probability upper bound on the excess risk. We give a differentially private stochastic Frank-Wolfe algorithm that is based on the variance reduction proposed in \cite{zhang2020one}. In our algorithm we combine a binary-tree-based variance reduction technique proposed by Asi et al.~for the $\ell_1$ setup \cite{AFKT:2021} with privacy-preserving gradient noise addition by our generalized Gaussian mechanism. Our algorithm has $T$ phases, where in each phase $t \in {1,\ldots,T}$ we construct a binary tree of depth $t$. Next, the algorithm traverses through the vertices of the tree  according to the Depth-First-Search (DFS) approach. % and allocate a set of samples to each vertex in the tree.
Following the notation in \cite{AFKT:2021}, vertices are denoted by $u_{t,s}$ where $s \in \{0,1\}^{\leq t}$ is the path to the vertex. Each vertex $u_{t,s}$ is associated with a parameter $x_{t,s}$, a gradient estimate $\tnab_{t,s}$, and a set of samples $S_{t,s}$. At a high level, each phase $t$ starts by computing a fresh estimate for the gradient of the population risk at root vertex (i.e. $s=\varnothing$) based on a large batch of samples. Next the vertices are traversed according to the Depth-First-Search approach and the gradient at each vertex is estimated using the samples of that vertex and the gradients along the path to the root. When a left child is visited, the parent node passes the parameter $x_{t,s}$ and gradient estimate $\tnab_{t,s}$ with no update. At every right child, the algorithm improves the gradient estimate using the estimate at the parent node. When a leaf vertex is visited in the DFS order, the algorithm updates the current iterate using the Frank-Wolfe step with the gradient estimate at the leaf. The vertices that are closer to the root are assigned more samples, i.e. each vertex is associated with a sample of size $\frac{b}{2^{j}}$, where $j$ is the depth of the vertex. Hence, the variance of the gradient estimate is reduced along the path to the root.
 
Following are some of the notations we use in this section for convenience. Let $\dfs(t)$ denote the DFS order of the vertices in a binary tree of depth $t$ (root not included). For $s \in \{0, 1\}^t$, let $\ell(s) \in [0,2^t-1]$  denote the integer whose binary representation is $s$. In the description of the algorithm, let  $x_{t,s}$ denote the iterates, where $t$ is the phase and $s \in \{0,1\}^{t}$ is the path from the root.

% We give a DP-SCO algorithm, which again is a variant of the variance-reduced stochastic Frank-Wolfe algorithm.
% Our algorithm in this $\ell_p$-setup differs from the polyhedral SFW (\cref{Alg:tree_PrivSCG}) in various ways: first, it uses gradient noise addition as privacy-preserving mechanism, perticularly our GG mechanism; %\footnote{On the positive side, this makes the algorithm more broadly applicable, but on the downside we lose the near dimension-independence in the risk.} 
% second, it uses the binary-tree-based variance reduction technique proposed by Asi et. al. \cite{AFKK21}, which helps in controlling the number of samples used to calculate the gradient estimate and third, the recursive gradient estimator is closer to the original SPIDER estimator \cite{Fang:2018}, which didn't use averaging factors $0<\rho<1$, and simply accumulates the gradient variations.

%In \cref{Alg:tree_PrivSCG}, we describe our Noisy Variance-Reduced Stochastic Frank-Wolfe algorithm $\cA_\scg$. %Our algorithm (\cref{Alg:tree_PrivSCG}) has $T$ phases, where in each phase $t \in {1,\ldots,T}$ we construct a binary tree of depth $t$ and allocate a set of samples to each vertex in the tree. 

\noindent Our algorithm, denoted by $\cA_\scg$, is provided in pseudocode in Algorithm~\ref{Alg:tree_PrivSCG}.

\begin{algorithm}[!h]
	\caption{$\cA_\scg$: Noisy Private Stochastic Frank-Wolfe Algorithm }
	\begin{algorithmic}[1]
		\REQUIRE Private dataset: $S =  (z_1,\ldots, z_n) \in \cZ^n$, %~$L_0$-Lipschitz, $L_1$-smooth w.r.t. $\norm{\cdot}$, convex loss function $f:\cX \times \cZ\mapsto \R$,
		~privacy parameters: $(\varepsilon,\delta)$, ~convex set: $\cX$, %\subseteq \B^d_{\norm{\cdot}_p}$ with $\norm{\cdot}$-diameter $M > 0$.
		~number of phases: $T$, ~batch size: $b$;

		\STATE Choose an arbitrary initial point $x_0 \in \cX$. Set $x_{0,\ell(0)}\rightarrow x_0$ 
		%\textcolor{red}{Need to set $x_{0,\ell(0)}\rightarrow x_0$}
		\STATE Set $\kappa:=\min\{1/(p-1), 2e \ln(d)\}$%\min\left\{\frac{1}{p-1}, 2e \ln(d)\right\}$.
		\FOR{$t =1$ to $T$}
		   \STATE Set $x_{t,\varnothing} = x_{t-1,\ell({t-1})}$ %\textcolor{red}{Is $\phi$ supposed to mean $\emptyset$? I would use the latter instead}
		   \STATE Draw a batch $S_{t,\varnothing}$ of $b$ data points without replacement from $S.$
		   \STATE Set $\sigma_{t,\varnothing}^2 := \frac{ 8\kappa L_0^2\log(1/\delta)}{b^2 \varepsilon^2}$.
		   \STATE Compute $\tnab_{t,\varnothing} = \frac{1}{b} \sum_{z \in S_{t,\varnothing}} \nabla f(x_{t,\varnothing},z) + \bg_{t,\varnothing} $, where  $\bg_{t,\varnothing} \sim  \cGG_{\|\cdot\|_+ }({\bf 0},\sigma_{t,\varnothing}^2)$. \label{stp:nablatree_0}
		   	
		   	\FOR{$u_{t,s} = \dfs[2^t]$}
		   	    \STATE Let $s=\hs c$, where $c \in \{0,1\}$.
		   	    \IF{$c=0$}
		   	        \STATE $\tnab_{t,s} = \tnab_{t,\hs}$.
		   	        \STATE $x_{t,s} =x_{t,\hs}$.
		   	    \ELSE
		   	         \STATE  Draw a batch $S_{t,s}$ of $\frac{b}{2^{|s|}}$ data points without replacement from $S.$   \label{stp:tree_scg_batch_t}
		   	         \STATE Set noise variance $\sigma_{t,s}^2 := \frac{ 128 \kappa  L_1^2 M^2 \log(1/\delta)}{b^2 \varepsilon^2}$.
		   	         \STATE Compute $\tDel_{t,s} = \frac{2^{|s|}}{b} \sum_{z \in S_{t,s}}\left(\nabla f(x_{t,s},z) - \nabla f(x_{t,\hs},z)\right)+ \bg_{t,s} $, where  $\bg_{t,s} \sim  \cGG_{\|\cdot\|_+ }({\bf 0},\sigma_{t,s}^2)$.  \label{stp:tree_scg_delta_t}  
		   	         \STATE $\tnab_{t,s} = \tnab_{t,\hs} + \tDel_{t,s} $.\label{stp:tree_scg_nabla_t}
		        \ENDIF
		        \IF{$|s|=t$}
		            \STATE Let $s^+$ be the next vertex in the DFS iteration.
		            \STATE Set $\eta_{t,s} := \frac{2}{2^{t-1}+\ell(s) + 1}$ 
		            \STATE Compute $v_{t,s} = \argmin\limits_{v \in \cX} \ip{\tnab_{t,s}}{v}$.
	            	\STATE $x_{t,s^+} \leftarrow (1-\eta_{t,s})x_{t,s} + \eta_{t,s} v_{t,s}$. 
	            \ENDIF
	       \ENDFOR
		\ENDFOR
		
		\STATE Output the final iterate $x_{2^T+1}$.
	\end{algorithmic}
	\label{Alg:tree_PrivSCG}
\end{algorithm}

\begin{theorem}[Privacy Guarantee of $\cA_\scg$]
Let  $\eta_{t,s} = \frac{1}{2^{t-1}+\ell(s) + 1} ~\forall s,t$. Then, Algorithm~\ref{Alg:tree_PrivSCG}  is $(\varepsilon, \delta)$-differentially private.
\end{theorem}
\begin{proof}
Let $S,S' \in \cZ^n$ be neighboring datasets. Let $\tnab_{t,s}$ and $\tnab'_{t,s}$ denote the private gradient estimates corresponding to $S$ and $S'$, respectively. For a phase $t$, let $\nabla_{t,\varnothing} =  \frac{1}{b} \sum_{z \in S_{t,\varnothing}} \nabla f(x_{t,\varnothing},z)$ be the gradient estimator for the root of the tree which has a sample size of $b$. %Then 
Note that the global $\dual{\cdot}$-sensitivity of $\nabla_{t,\varnothing}$ is bounded by $\frac{2L_0}{b}$. Hence, by Corollary~\ref{cor:GG_mech_priv}
%\cnote{should be Corollary~\ref{cor:GG_mech_priv} instead?}
%Corollary~\ref{cor:GG_mech_priv} 
we obtain that Step~\ref{stp:nablatree_0} in Algorithm~\ref{Alg:tree_PrivSCG} is $(\varepsilon,\delta)$-differentially private.. 

Next, for the iterate $x_{t,s}$, let $S_{t,s}$ denote the mini-batch given in Step~\ref{stp:tree_scg_batch_t} of Algorithm~\ref{Alg:tree_PrivSCG}.
%This set of size $\frac{b}{2^{|s|}}$ is used to calculate the gradient estimate for at most $2^{t-|s|}$ 
% For iteration $t \in [\sqrt{n}]$, let $B_t$ denote the mini-batch given in Step~\ref{stp:scg_batch_t} in Algorithm~\ref{Alg:PrivSCG}, and 
Let  $\Delta_{t,s} = \frac{2^{|s|}}{b} \sum_{z \in S_{t,s}}\left(\nabla f(x_{t,s},z) - \nabla f(x_{t,\hs},z)\right)$.
Also, let $S_{t,s}',\Delta_{t,s}'$ denote the corresponding quantities for Algorithm $\cA_\scg$ when the input dataset is $S'$. Suppose that $S_{t,s}$ and $S_{t,s}'$ differ in at most one data point, say $z_{i^*} \neq z_{i^*}'$. Then
$$\textstyle \dual{\Delta_{t,s} - \Delta_{t,s}'} = \frac{2^{|s|}}{b}\|\nabla f(x_{t,s},z_{i^*}) - \nabla f(x_{t,\hs},z_{i^*}) - (\nabla f(x_{t,s},z_{i^*}') - \nabla f(x_{t,\hs},z_{i^*}'))\|_\ast.$$

Here for the current iterate $x_{t,s}$, $x_{t,\hs}$ denotes the iterate of its parent vertex. Recall that whenever the algorithm visits a leaf node, it applies a Frank-Wolfe step to calculate the next iterate and puts its value in the next vertex in the DFS order. %Note that here the vertex $u_{t,s}$ is the right son of the vertex  $u_{t,\hs}$, 
Hence, the total number of (consecutive) iterates between $x_{t,s}$ and $x_{t,\hs}$ is at most the number of leaf vertices visited between these two vertices. In particular, these are leaf vertices that are descendants of the vertex $u_{t,s}$, which is $2^{t-|s|}$, where $|s|$ denotes the depth of the vertex.
%\cnote{I think it is a good idea to add more details here. In particular, $2^{t-|s|}$ arises because it it is the number of iterations between $x_{t,s}$ and $x_{t,\hs}$. But if stepsizes are decreasing, should't we use the larger stepsize? Meaning $\eta_{t,s^{\prime}}$} 
By the smoothness of $f$ w.r.t. $\norm{\cdot}$, the global $\dual{\cdot}$ sensitivity of $\Delta_{t,s}$ is bounded by $\frac{2^{|s|} 2^{t-|s|}\eta_{t,\hs} L_1 M}{b}$. Given the setting of $\eta_{t,\hs}$, we obtain that the global $\dual{\cdot}$ sensitivity of $\Delta_{t,s}$ is bounded by $\frac{8 L_1 M}{b}$. %\cnote{Note there is at least a factor 2 missing, since for both $S$ and $S^{\prime}$ we need to do this estimate on the gradient difference norm. On top of this, given the choice of stepsizes, I believe another factor of 4 is needed in the sensitivity bound. So, overall we need an extra factor of 8 (i.e., a 64 factor in the variance). But please check this carefully, as I may be doing some conservative estimates.} \anote{The factors were missing in the variances. Corrected the factors in the variances } \cnote{Looks good!}
Again, using Corollary~\ref{cor:GG_mech_priv} we have that Step~\ref{stp:tree_scg_delta_t} in Algorithm~\ref{Alg:tree_PrivSCG} is $(\varepsilon,\delta)$-differentially private.  

Note that at any given iterate $x_{t,s}$, the gradient estimate $\tnab_{t,\hs}$ from  the parent iterate $x_{t,\hs}$ is already computed privately. Since differential privacy is closed under post-processing, the current iterate $x_{t,s}$ is $(\varepsilon,\delta)$-DP. 
%Since the batches of the dataset used in different iterations are disjoint, 
The sample set $S_{t,s}$ is used to calculate the gradient estimate for at most $2^{t-|s|}$ contiguous sequence of iterates i.e. the leaves of the tree of depth $t$, but the batches of the dataset used in different iterations are disjoint. Hence, by parallel composition, Algorithm $\cA_\scg$ is $(\varepsilon,\delta)$-differentially private.  
% \textcolor{red}{Do we actually need strong composition, if we are not re-using the batches? Notice that at each step noise is added, so privatization is guaranteed by parallel composition. Corrected}
\end{proof}

\begin{theorem}[Accuracy Guarantee of $\cA_\scg$]\label{thm:tree_SCGAlgAcc}
For $p \in (1,2)$, consider the $\ell_p$-setup of DP-SCO. Setting $\kappa = \min\big\{\frac{1}{p-1}, 2e \ln(d)\big\}$, $T= \frac{\log(n)}{2}$, and $b=\frac{4n}{\log^2(n)}$, 
for any distribution ${\cal D}$ supported on $\cZ$ and any $\beta\in(0,1)$, $\cA_\scg$  satisfies with probability at least $1-\beta$: %\cnote{Following a comment by Raef, the term in red below should be removed}
%\ex{S \sim \cD^n, \cA_\scg}{F_\cD(x^\prv) -  F_\cD(x^*)} 
$$ \textstyle \Hp[\cA_\scg] = O\left(\!\left(\sqrt{\kappa}+\sqrt{\log(\frac{n}{\beta})}\right)\left(L_0M +L_1M^2\right) \left(\frac{\log^2(n)}{\sqrt{n}} \!+\! \frac{\log^3(n)\sqrt{\kappa d\log(1/\delta)}}{n\varepsilon}\right)\!\right)\!.$$
% where $\Hp[\cA_\scg]  = F_\cD({\cA_\scg}(S)) -  F_\cD^{\ast}$. 
% \cnote{Why not defining this in the preliminaries? In fact, we can even replace the definition of excess population risk by this one.}\anote{Defined $\Hp[\cA]  = F_\cD({\cA}(S)) -  F_\cD^{\ast}$ in preliminaries.}
% 	$${\cal R}_{\cal D}[{\cal A}_{\scg}] =O\left(\left(L_0M +L_1M^2\sqrt{\kappa}\right) \cdot \left(\log(n)\sqrt{\frac{\kappa}{n}} + \frac{\log^2(n)\sqrt{\kappa d\log(1/\delta)}}{n\varepsilon}\right) \right).$$
	%where $x^* = \argmin\limits_{x \in \cX} F_\cD(x),$ and $\kappa = \min\left\{\frac{1}{p-1}, 2 \ln(d)\right\}$.
\end{theorem}

In order to prove the excess risk guarantee of Algorithm~\ref{Alg:tree_PrivSCG}, we start by proving a recursive bound on the first moment of the gradient estimator.

% \begin{lemma}\label{lem:tree_privSCG_grad_var}
% Let $\cD$ be a distribution over $\cZ$, and $S \sim \cD^n$ be the input to Algorithm $\cA_\scg$. For $t \in [1,T]$, $s \in \{0, 1\}^t$, and $\kappa = \min\left\{\frac{1}{p-1}, e^2 \ln(d)\right\}$, let  $\eta_{t,s} = \frac{1}{2^{t-1}+\ell(s) + 1}$. Then, the recursive gradient estimate $\tnab_{t,s}$ satisfies 
% 	\begin{align*}
% 	    \ex{S\sim {\cal D}^n}{\|\nabla F_{\cD}(x_{t,s})-\tnab_{t,s}\|_{*}} &\leq   2 L_0 \sqrt{\frac{\kappa}{b}} +  2 L_1 M \sqrt{\frac{\kappa}{b}} \\
%     &+  \frac{2 \kappa L_0 \sqrt{d \log(1/\delta)}}{b\varepsilon} +  \frac{2 \kappa  L_1 M \sqrt{td \log(1/\delta)}}{b\varepsilon}.
% 	\end{align*}
% \end{lemma}
% \begin{lemma}\label{lem:tree_privSCG_grad_var}
% Let $\cD$ be a distribution over $\cZ$, and $S \sim \cD^n$ be the input to Algorithm $\cA_\scg$. For $t \in [1,T]$, $s \in \{0, 1\}^t$, and $\kappa = \min\left\{\frac{1}{p-1}, e^2 \ln(d)\right\}$, let  $\eta_{t,s} = \frac{1}{2^{t-1}+\ell(s) + 1}$. Then, the recursive gradient estimate $\tnab_{t,s}$ satisfies 
% 	\begin{align*}
% 	    \ex{S\sim {\cal D}^n}{\|\nabla F_{\cD}(x_{t,s})-\tnab_{t,s}\|_{*}} &\leq   2 L_0 \sqrt{\frac{\kappa}{b}} +  2 L_1 M \sqrt{\frac{\kappa}{b}} \\
%     &+  \frac{2 \kappa L_0 \sqrt{d \log(1/\delta)}}{b\varepsilon} +  \frac{2 \kappa  L_1 M \sqrt{td \log(1/\delta)}}{b\varepsilon}.
% 	\end{align*}
% \end{lemma}
%\anote{union over $2^t$}
\ifarxiv
\begin{lem}\label{lem:tree_privSCG_grad_var}
Let $\cD$ be a distribution over $\cZ$, $S \sim \cD^n$ be the input to Algorithm~\ref{Alg:tree_PrivSCG}, and $\kappa = \min\left\{\frac{1}{p-1}, 2e \ln(d)\right\}$. %For $t \in [1,T]$, and $s \in \{0, 1\}^t$, 
%and  $\eta_{t,s} = \frac{1}{2^{t-1}+\ell(s) + 1}$. 
Then, for any $\beta \in (0,1)$, w.p. at least $1-\beta$, for all $t \in [1,T]$ and $s\in\{0,1\}^t$: \cnote{Please check, since it was previously stated only for the leaves $s\in\{0,1\}^t$}\anote{Corrected}
%the recursive gradient estimator $\tnab_{t,s}$ satisfies w.p.~$1-\beta$:
$$\textstyle \|\tnab_{t,s}-\nabla F_{\cal D}(x_{t,s})\|_{\ast} \leq 2\left(\sqrt{e\kappa}+\sqrt{3\log(\frac{2^t}{\beta})}\right)\left(\frac{L_0 + L_1M}{\sqrt{b}} \!+\! \frac{ (2L_0 +8 L_1 Mt) \sqrt{\kappa\,d \log(1/\delta)}}{b \varepsilon} \right).$$
\end{lem}

\else

\begin{lemma}\label{lem:tree_privSCG_grad_var}
Let $\cD$ be a distribution over $\cZ$, $S \sim \cD^n$ be the input to \cref{Alg:tree_PrivSCG}, and $\kappa = \min\left\{\frac{1}{p-1}, 2e \ln(d)\right\}$. %For $t \in [1,T]$, and $s \in \{0, 1\}^t$, 
%and  $\eta_{t,s} = \frac{1}{2^{t-1}+\ell(s) + 1}$. 
Then, for any $\beta \in (0,1)$, w.p. at least $1-\beta$, for all $t \in [1,T]$ and $s\in\{0,1\}^t$: \cnote{Please check, since it was previously stated only for the leaves $s\in\{0,1\}^t$}\anote{Corrected}
%the recursive gradient estimator $\tnab_{t,s}$ satisfies w.p.~$1-\beta$:
$$\textstyle \|\tnab_{t,s}-\nabla F_{\cal D}(x_{t,s})\|_{\ast} \leq 2\left(\sqrt{e\kappa}+\sqrt{3\log(\frac{2^t}{\beta})}\right)\left(\frac{L_0 + L_1M}{\sqrt{b}} \!+\! \frac{ (2L_0 +8 L_1 Mt) \sqrt{\kappa\,d \log(1/\delta)}}{b \varepsilon} \right).$$
\end{lemma}

\fi

\begin{proof}
Consider any phase $t\geq 1$ of $\cA_{\scg}$. In Step~\ref{stp:nablatree_0} and Step~\ref{stp:tree_scg_delta_t} of $\cA_\scg$, we add GG noise $\bg_{t,s}$ with $\sigma_{t,\varnothing}^2 := \frac{ 8\kappa L_0^2\log(1/\delta)}{b^2 \varepsilon^2}$ and $\sigma_{t,s}^2 := \frac{128 \kappa  L_1^2 M^2 \log(1/\delta)}{b^2 \varepsilon^2}$, respectively. Then, by property (b) in Proposition~\ref{prop:basic_prop_GG} we have that for $\nu_{t,s} = \sigma_{t,s}\sqrt{d+2}$, ~$\ex{}{\exp\{\|\bg_{t,s}\|^2_\ast/\nu^2_{t,s}\}} \leq \exp(1).$ 
Recall that every right child computes the gradient estimate using the estimate at the parent vertex. Hence, the gradient estimate is updated using the gradients along
the path to the root and the total number of gradient estimate updates in phase $t$ is at most $2^{t}$. Now given, $|S_{t,\varnothing}| = b$, $|S_{t,s}|=\frac{2^{|s|}}{b}$, and the recursive gradient estimator in Steps~\ref{stp:tree_scg_delta_t}-\ref{stp:tree_scg_nabla_t},  Lemma~\ref{lem:hp_var_grad_est_gen} and taking union bound over all $(t,s)$ the following event has probability $\geq 1-\beta$ (below $C_{\beta}:=\sqrt{e\kappa}+\sqrt{3\log(\frac{2^t}{\beta})}$)
%\anote{using union bound over all $t,s$ the following event has prob..}
\begin{align*}
\textstyle \|\tnab_{t,s}-\nabla F_{\cal D}(x_{t,s})\|_{\ast}  \textstyle \leq  2C_{\beta}\Big(\frac{ L_0}{\sqrt{b}} +  L_1M\eta_{t,s} {\displaystyle\sum_{j<t}}\frac{2^{j/2}}{\sqrt{b}}
 \textstyle+ \frac{\big( 2L_0 + \sum\limits_{j<t} 8L_1 M \big)  \sqrt{\kappa\,d\log(1/\delta)}}{b \varepsilon}\Big).
\end{align*}
Using the above bounds and the setting of  $\eta_{t,s} = \frac{1}{2^{t-1}+\ell(s) + 1}$, we can finally arrive at 
\begin{align*}
\textstyle \|\tnab_{t,s}-\nabla F_{\cal D}(x_{t,s})\|_{\ast} \leq 2\left(\sqrt{e\kappa}+\sqrt{3\log(\frac{2^t}{\beta})}\right)\left(\frac{L_0 + L_1M}{\sqrt{b}} \!+\! \frac{ (2L_0 + 8L_1 Mt) \sqrt{\kappa\,d \log(1/\delta)}}{b \varepsilon} \right).
\end{align*}
\end{proof}

\begin{proof}[Proof of Theorem~\ref{thm:tree_SCGAlgAcc}]
Here we use an equivalent representation $m= 2^{t-1}+\ell(s)$ for a leaf vertex $u_{t,s}$, where $\ell(s)$ is the integer whose binary representation is $s$. Note that a contiguous sequence of iterates in phase $t$ are the leaves of the tree of depth $t$ from left to right. Since $t \leq T$, we get that $m \in \{1,\ldots,2^T\}$. Hence, for $m \in [1,2^T]$, by smoothness and convexity of $F_\cD$, we have 
	\begin{align*}
 \textstyle	F_\cD(x_{m+1}) &\leq F_\cD(x_{m}) + \ip{\nabla F_\cD(x_m)}{x_{m+1} - x_m} + \frac{L_1}{2}\|x_{m+1} - x_m\|^2\\
& \hspace{-0.4cm}\textstyle \leq  F_\cD(x_{m}) +  \ip{\nabla F_\cD(x_m) - \tnab_m}{x_{m+1}\! -\! x_m} + \frac{L_1M^2\eta_m^2 }{2} +   \ip{\tnab_m}{x_{m+1} \! -\! x_m}\\
& \hspace{-0.4cm} \textstyle\leq F_\cD(x_m) +  2\eta_m M \dual{\nabla F_\cD(x_m) - \tnab_m} + \frac{L_1M^2\eta_m^2 }{2} + \eta_m \ip{\nabla F_\cD(x_m)}{x^* -x^t}.
% 	&\textstyle \leq F_\cD(x^{t}) + 2 \eta M \dual{\nabla F_\cD(x^t) - \bfd_t} + \frac{L_1\eta^2 M^2}{2} + \eta \ip{\nabla F_\cD(x^t)}{x^* -x^t} + \eta \alpha_t\\
\end{align*}
Thus, we get
\begin{align*}
	\textstyle F_\cD(x_{m+1}) - F_\cD^{\ast} \leq (1-\eta_m)\left(F_\cD(x_m)-F_\cD^{\ast}\right) + 2\eta_{m} M \dual{\nabla F_\cD(x_m)- \tnab_m} + \frac{L_1M^2\eta_m^2 }{2}.
\end{align*}

\noindent Letting $\G_m = F_\cD(x_{m})  - F_\cD^{\ast}$, we get the following recursion:
	\begin{align*}
	\G_{m+1} &\leq (1-\eta_m)\G_m + 2\eta_m M \dual{\nabla F_\cD(x_m)- \tnab_m} + \frac{L_1M^2\eta_m^2 }{2}.
\end{align*}

%and let $\bq_m = \dual{\nabla F_\cD(x_m) - \tnab_m}$.  
% Hence, taking expectation and by \cref{lem:tree_privSCG_grad_var} 	we get
% 	\begin{align*}
% 	\ex{}{\G_{m+1}} &\leq (1-\eta_m)\ex{}{\G_m} + 2\eta_m M \ex{}{\dual{\nabla F_\cD(x_m)- \tnab_m}} + \frac{L_1 \eta_m^2 M^2}{2}\\
% 	 &\leq (1-\eta_m)\ex{}{\G_m} + 4\eta_m L_0M \sqrt{\frac{\kappa}{b}} + \frac{4\kappa \eta_m L_0 M \sqrt{d\log(1/\delta)} }{b\varepsilon} + 4\kappa \eta_m L_1M^2  \sqrt{\frac{\kappa}{b}}  \\
% 	 &+ \frac{4\kappa \eta_m L_1M^2 \sqrt{d \log(m)\log(1/\delta)}}{b\varepsilon}+ \frac{L_1 \eta_m^2 M^2}{2}
% \end{align*}

\noindent Note that here $\eta_m= \frac{2}{m+1}$. Now letting $C_m = 2 M \dual{\nabla F_\cD(x_m)- \tnab_m}+ \frac{L_1M^2\eta_m }{2}$, and expanding the above recursion we have:
\begin{align*}
\G_{2^T+1}  &\leq \prod_{m=1}^{2^T}(1-\eta_m)\G_0 + \sum_{m=1}^{2^T} \eta_m C_m \prod_{i>m} (1-\eta_i)
\leq  \sum_{m=1}^{2^T} \eta_m C_m \frac{(m-1)m}{2^T(2^T+1)}\\
&\leq  \sum_{m=1}^{2^T} \eta_m C_m \frac{m^2}{2^{2T}}.
%&\leq 4 L_0M \sqrt{\frac{\kappa}{b}} + \frac{4\kappa L_0 M \sqrt{d\log(1/\delta)} }{b\varepsilon} +  4\kappa \eta_m L_1M^2  \sqrt{\frac{\kappa}{b}}\\ &+ \frac{4\kappa L_1M^2 \sqrt{d~T~\log(1/\delta)}}{b\varepsilon}+ \frac{L_1 M^2}{2^T}
\end{align*}
Hence, using Lemma~\ref{lem:tree_privSCG_grad_var} and taking union bound over all $m \in [1,2^T]$, we obtain the following bound with probability at least $1 -\beta$
\begin{align*}
\G_{2^T+1} &\textstyle \leq 4\left(\sqrt{e\kappa}+\sqrt{3\log(\frac{2^{2T}}{\beta})}\right)\Big(\frac{L_0M \!+\! L_1M^2}{\sqrt{b}} +  \frac{ (2L_0M \!+\! 8L_1 M^2T) \sqrt{\kappa d \log(1/\delta)}}{b \varepsilon}\Big) +  \frac{L_1 M^2}{2^T}.
\end{align*}

\noindent Setting $T = \frac{\log(n)}{2}$, and $b=\frac{4n}{\log^2(n)}$ we get w.p.~$\geq 1-\beta$ (below $C_{\beta}:=\sqrt{e\kappa}+\sqrt{3\log(\frac{n}{\beta})}$) 
\begin{align*}
\G_{2^T+1} &\textstyle \leq 4C_{\beta}\log^2(n)\Big(\frac{(L_0M + L_1M^2)}{\sqrt{n}} +\frac{ (2L_0M \!+\! 8L_1 M^2\log(n)) \sqrt{\kappa d \log(1/\delta)}}{n \varepsilon}\big)+  \frac{L_1 M^2}{\sqrt{n}}.
\end{align*}
% Now setting $T = \log(n)$, and $b=\frac{n}{\log^2(n)}$ we get
% \begin{align*}
% 	\ex{}{\G_{2^T+1}} &\leq   4 L_0 M \log(n) \sqrt{\frac{\kappa}{n}} + \frac{ 4\kappa L_0 M \log^2(n) \sqrt{d \log(1/\delta)}}{n~\varepsilon}  \\
% 	&+  \frac{ \kappa L_1 M^2 \log(n)}{\sqrt{n}} + \frac{ \kappa L_1 M^2 \log^{5/2}(n) \sqrt{ d \log(1/\delta)}}{n~\varepsilon} + \frac{L_1 M^2}{2n}. 
% \end{align*}

\noindent Hence, we conclude that w.p. $\geq 1-\beta$ 
%\ex{S \sim \cD^n, \cA_\scg}{F_\cD(x^\prv) -  F_\cD(x^*)} 
$$ \textstyle F_\cD({\cA_\scg}(S)) -  F_\cD^{\ast} =O\left(\!
C_{\beta}
%\left(\sqrt{\textcolor{red}{e}\kappa}+\sqrt{\log(\frac{n}{\beta})}\right)
\left(L_0M +L_1M^2\right) \left(\frac{\log^2(n)}{\sqrt{n}} \!+\! \frac{\log^3(n)\sqrt{\kappa d\log(1/\delta)}}{n\varepsilon}\right)\!\right).
 $$
	
\end{proof}

%\anote{Add small paragraph on lower bound to show optimality of bounds}
\subsection{Lower Bound for DP-SCO in the $\ell_p$ setup for $1< p < 2$}\label{subsec:lower_bound_main}
We provide lower bounds on the excess risk for DP-SCO in the $\ell_p$ setting for $1<p<2$. In our argument, we first prove a lower bound on DP-ERM, then use the reduction in \cite[Appendix C]{bassily2019private} to assert that essentially the same lower bound (up to a logarithmic factor in $1/\delta$) holds for DP-SCO. A crucial step in our proof to construct a lower bound for DP-ERM is to transform a lower bound on the $\ell_p$ distance to the minimizer into a lower bound on the excess risk. We remark that this step requires new tools than what is readily available in the Euclidean setting (considered in the lower bound of \cite{BST}). In particular, it relies on the strong convexity property of $\ell_p$ spaces for $1<p<2$ \cite{ball1994sharp}. Next, we prove a reduction from DP-ERM in this
setting to privately estimating 1-way marginals using the fingerprinting
code argument from \cite{bun2018fingerprinting}. Our final lower bound for DP-SCO follows from combining the lower bound on DP-ERM with the non-private $\Omega(1/\sqrt{n})$ lower bound for SCO when $1<p<2$ \cite{nemirovsky1983problem}. Below, we formally state our lower bound for DP-SCO. 
% \ifarxiv For the full details of our construction and the statement of the lower bound for DP-ERM please refer to the full version of our conference paper~\cite{BGN:2021}. 
% \else
We defer the full details of our construction and the statement of the lower bound for DP-ERM to Appendix~\ref{appx:lower_bound}. %\fi

\begin{theorem}[Lower Bound for DP-SCO for $p\in (1, 2)$]\label{thm:lower_bound_main}
Let $p \in (1, 2)$ and $n, d\in \mathbb{N}$. Let $\varepsilon >0$ and $0< \delta <\frac{1}{n^{1+{\Omega(1)}}}$. Let $\cX=\B_p^d$, where $\B_p^d$ is the unit $\ell_p$ ball in $\R^d$, and $\cZ=\{-\frac{1}{d^{1/q}}, \frac{1}{d^{1/q}}\}^d$, where $q=\frac{p}{p-1}$. There exists a distribution $\cD$ over $\cZ$ such that for any $(\varepsilon, \delta)$-DP-SCO algorithm $\cA:\cZ^n\rightarrow \cX$, we have 
$${\cR}_{\cD}[\cA]=\tilde\Omega\left(\max\left(\frac{1}{\sqrt{n}}, (p-1)\frac{\sqrt{d}}{\varepsilon n}\right)\right).$$
\end{theorem}

%% file: SCOPrivFW.tex
\section{Private Stochastic Frank-Wolfe with Variance Reduction for Polyhedral Setup }\label{sec:SCOPrivFW}

In this section, we consider DP-SCO in the {\em polyhedral} setup. Let $K$ be a positive integer, and consider $(\bE,\|\cdot\|)$ a normed space, where the unit ball of the norm, ${\cal B}_{\|\cdot\|}=\mbox{conv}({\cV})$ is a polytope with at most $K$ vertices. %Under %, with set of vertices ${\cal A}$
Further, the feasible set ${\cal X}$ is a polytope with at most $K$-vertices and $\|\cdot\|$-diameter $M>0$. Notice that since the norm its polyhedral, its dual norm is also polyhedral. Moreover, $(\bE,\|\cdot\|_{\ast})$ is $O(\ln K)$-regular.

%For the algorithm, we also require the loss function $f:\cX \times \cZ\mapsto \R$ to be convex, Lipschitz, and $L_1$-smooth w.r.t.~$\|\cdot\|$. Namely, for any $z \in \cZ$ and any $x,x' \in \cX$ we assume that:
%\begin{align*}
%     \dual{\nabla f(x,z)} \leq L_0, \quad \text{and}\quad \dual{\nabla f(x,z)- \nabla f(x',z)} \leq L_1\norm{x - x'}
%\end{align*}
%\textcolor{red}{CG: We don't need this anymore. It's added in the preliminaries.}

%\textcolor{red}{I am presenting above the most general setting we can address with the private \FW algorithm. We can also simply consider $\|\cdot\|_1$, if you prefer.}

%We consider the setting where the objective function $f:\cX \times \cZ\mapsto \R$ is convex, Lipschitz, and smooth.
We describe another variant of the variance-reduced stochastic Frank-Wolfe algorithm that is based on the variance reduction approach proposed in  \cite{zhang2020one}. This algorithm differs from the noisy SFW algorithm (Algorithm~\ref{Alg:tree_PrivSCG}) in various ways: first since the feasible set is polyhedral, the linear optimization oracle at each iteration is privatized using the report noisy max mechanism  \cite{DR14,bhaskar2010discovering}; second,  it uses a single data point per iteration, allowing for larger number of iterations; and third, the recursive gradient estimator in this algorithm is the one used in \cite{zhang2020one}. We construct the private unbiased gradient estimator as follows:
\begin{equation}\label{eqn:priv_rec_grad}
    \bfd_t \triangleq (1-\rho_t)\left(\bfd_{t-1} + \Delta_t \right) + \rho_t \nabla f(x^t,z_t),
\end{equation}
where, $\Delta_t(z_t) \triangleq \nabla f(x^t,z_t) - \nabla f(x^{t-1},z_t)$ is the gradient variation for a given sample point $z_t \in \cZ$ and $x^t,x^{t-1} \in \cX.$ Here, for all $t$ we choose $\rho_t = \eta$, where $\eta$ is the step size. We compute a private version of $\bfd_t$ via the Report Noisy Max mechanism \cite{DR14,bhaskar2010discovering}. Given step size $\eta$, gradient estimate $\bfd_t$, set of vertices $\cV$, and global sensitivity of $\ip{v}{\bfd_t}$, that we call $s_t$, we have the following private Frank-Wolfe update step:
\begin{align*}
\textstyle x^{t+1} &= (1 - \eta) x^{t}  + \eta~v_t, ~~~~~~~\quad\text{where} \\ v_t &= \argmin_{v \in \cV}{\{\ip{v}{\bfd_t} + u^t_v\}},  \quad u^t_v \sim \lap(2s_t \sqrt{n \log(1/\delta)}/\varepsilon).
\end{align*}

\noindent We provide a pseudocode of the algorithm in 
Algorithm~\ref{Alg:PrivSFW}.
%In \cref{Alg:PrivSFW} %(denoted as $\cA_\polyfw$), 
%we describe our Private Polyhedral Stochastic Frank-Wolfe Algorithm. 

\begin{algorithm}[!h]
	\caption{$\cA_\polyfw$: Private Polyhedral Stochastic Frank-Wolfe Algorithm}
	\begin{algorithmic}[1]
		\REQUIRE Private dataset $S =  (z_1,\ldots z_n) \in \cZ^n$, %~$L_0$-Lipschitz, $L_1$-smooth, convex loss function $f:\cX \times \cZ\mapsto \R$, 
		~privacy parameters $(\varepsilon,\delta)$, ~ polyhedral set $\cX$ with a set of $K$ vertices $\cV = (v_1,\ldots,v_K)$%, where the $\norm{\cdot}$-diameter 
		%w.r.t the associated norm $\norm{\cdot}$ 
		%bounded by $M > 0$.
		
		\STATE Set step size $\eta := \frac{\log\left(n/\log(K)\right)}{n}$
		
		%\STATE Let $\pi$ be a uniformly random permutation of $[n]$.\label{stp:permute}
		%\STATE $\piS \leftarrow \left(z_{\pi(1)},z_{\pi(2)},\ldots, z_{\pi(n)} \right)$
		\STATE Choose an arbitrary initial point $x^0 \in \cX$
		\STATE Let $B_0 = (z^0_1,\ldots , z^0_{n/2})$ be an initial batch of $\frac{n}{2}$ data points from  $S$
		\STATE Compute ${\bfd_0} = \frac{2}{n} \sum_{i=1}^{n/2} \nabla f(x^0,z^0_i)$ \label{stp:d0}
		
		\STATE $v_0 = \argmin\limits_{v \in \cV}{\{\ip{v}{\bfd_0} + u^0_v\}}$, where $u^0_v \sim \lap\left(\frac{4L_0 M \sqrt{\log(1/\delta)}}{\varepsilon~\sqrt{n}} \right)$\label{stp:v_0}
		%\lap\left(\frac{4M(L_1 +L_0)\eta \sqrt{2 \log(1/\delta)}}{\varepsilon}\right),$ where $\lap(\lambda) \sim \frac{1}{2\lambda} e^{-|x|/\lambda}.$\label{stp:gamma}
		\STATE $x^{1} \leftarrow (1-\eta)x^{0} + \eta v_0$
		\STATE Let $\hS = (z_1,\ldots , z_{n/2})$ be the remaining $\frac{n}{2}$ data points in $S$ that are not in $B_0$
		
		\FOR{$t =1$ to $\frac{n}{2}$}
		
		\STATE Set $s_t := \max\left\{(1-\eta)^t \cdot \frac{2L_0 M}{n}, 2\eta~( L_1M^2 + L_0 M) \right\}$   \label{stp:sens_st}  
		%shuffle before one pass 
		%\STATE Sample a point $z_t$ from $\hS$ uniformly without replacement.\label{stp:sample}
		\STATE Compute $\Delta_t(z_t) = \nabla f(x^t,z_t) - \nabla f(x^{t-1},z_t) $
		\STATE $\bfd_t = (1-\eta)\left(\bfd_{t-1} + \Delta_t(z_t) \right) + \eta \nabla f(x^t,z_t)$\label{stp:grad_est}
		
		\STATE $\forall v \in \cV, \gamma_v \leftarrow \ip{v}{\bfd_t} + u^t_v$, where $u^t_v \sim \lap\left(\frac{2s_t \sqrt{n \log(1/\delta)}} {\varepsilon}\right)$\label{stp:gamma}
		%\lap\left(\frac{4M(L_1 +L_0)\eta \sqrt{2 \log(1/\delta)}}{\varepsilon}\right),$ where $\lap(\lambda) \sim \frac{1}{2\lambda} e^{-|x|/\lambda}.$\label{stp:gamma}
		\STATE Compute $v_t = \argmin_{v \in \cV}{\gamma_v}$\label{stp:argmax}
		\STATE $x^{t+1} \leftarrow (1-\eta)x^{t} + \eta v_t$ \label{stp:update}
		\ENDFOR
		\STATE Output $x^\prv =x^{n/2 + 1}$
	\end{algorithmic}
	\label{Alg:PrivSFW}
\end{algorithm}
%\textcolor{red}{CG: Lines 5 and 12 of the Algo I find very confusing, can we say instead $(u_k)_{k\in[K]}$ are our noise addition r.v.s?}

Next, we will briefly describe our analysis for the privacy guarantee and  excess population risk of Algorithm~\ref{Alg:PrivSFW}. 
First, we upper bound the global $\dual{\cdot}$ sensitivity of $\bfd_t$. By Step~\ref{stp:d0} in Algorithm~\ref{Alg:PrivSFW}, we have that %it is clear that, 
for $t=0$, the $\|\cdot\|_{\ast}$ sensitivity of $\bfd_0$ is at most  $\frac{2L_0}{n}$. For $t \geq 1$, by expanding the recursion \ref{eqn:priv_rec_grad} and using that $f(\cdot,z)$ is $L_0$-Lipschitz and $L_1$-smooth w.r.t $\norm{\cdot}$, we obtain the following lemma: 
\ifarxiv
\begin{lem}\label{lem:privSFW_sens}
	For Algorithm~\ref{Alg:PrivSFW} (Algorithm $\cA_\polyfw$), 
	%with step size $\eta := \frac{2 \log\left(n/\log(K) \right)}{n}$.
	%, and suppose that the loss function $f$ is convex, $L_0$-Lipschitz, and $L_1$-smooth w.r.t. $\norm{\cdot}$. 
	%For any $v \in \cV$, 
	let $s_t$ be the global sensitivity of $\ip{v}{\bfd_t}$, namely $ s_t = \max_{v\in\cV}\max_{S\simeq S'}|\ip{v}{\bfd_t-\bfd_t'}|$ . Then %for any iteration $t$, 
	$$s_t \leq \max\Big\{(1-\eta)^t \cdot \frac{2L_0 M}{n}, 2\eta~( L_1M^2 + L_0 M) \Big\} \qquad(\forall t\in[n/2]).$$ 
\end{lem}

\else
\begin{lemma}\label{lem:privSFW_sens}
	For \cref{Alg:PrivSFW} (Algorithm $\cA_\polyfw$), 
	%with step size $\eta := \frac{2 \log\left(n/\log(K) \right)}{n}$.
	%, and suppose that the loss function $f$ is convex, $L_0$-Lipschitz, and $L_1$-smooth w.r.t. $\norm{\cdot}$. 
	%For any $v \in \cV$, 
	let $s_t$ be the global sensitivity of $\ip{v}{\bfd_t}$, namely $ s_t = \max_{v\in\cV}\max_{S\simeq S'}|\ip{v}{\bfd_t-\bfd_t'}|$ . Then %for any iteration $t$, 
	$$s_t \leq \max\Big\{(1-\eta)^t \cdot \frac{2L_0 M}{n}, 2\eta~( L_1M^2 + L_0 M) \Big\} \qquad(\forall t\in[n/2]).$$ 
\end{lemma}
\fi
Next, by the privacy guarantee of the Report Noisy Max mechanism \cite{DR14,bhaskar2010discovering}, Step~\ref{stp:v_0} of Algorithm~\ref{Alg:PrivSFW} is $\frac{\varepsilon}{\sqrt{n\log(1/\delta)}}$-DP  and Steps~\ref{stp:sens_st}-\ref{stp:update} are $\frac{\varepsilon}{\sqrt{n\log(1/\delta)}}$-DP. Thus, by the advanced composition theorem (Lemma~\ref{lem:adv_comp}), Algorithm  $\cA_\polyfw$ is $(\varepsilon,\delta)$-DP.

Next, we turn to the analysis of the excess risk of $\cA_\polyfw$. We first show that the variance of the gradient estimator is bounded with high probability. In Algorithm~\ref{Alg:PrivSFW}, we use a large batch in the first iteration, i.e.~$|B_0|= \frac{n}{2}$, and use a single data point to compute the gradient estimate in the remaining $n/2$ iterations. Recall that, here we do not add any noise to the gradient estimator. Hence, plugging the values for $B_0,\ldots,B_t$, and $\bg_t$ in Lemma~\ref{lem:hp_var_grad_est_gen} and for all $t\in [0,\frac{n}{2}]$ iterations, we obtain the following lemma.

\ifarxiv
\begin{lem}\label{lem:privSFW_grad_var}
Let $\cD$ be any distribution over $\cZ$. Let $S\sim\cD^n$ be the input dataset of $\cA_\polyfw$ (Algorithm~\ref{Alg:PrivSFW}). For any $\beta \in(0,1)$, 
%the recursive gradient estimator $\bfd_t$  
with probability $\geq 1-\beta$, for all $t \in [0,\frac{n}{2}]$\rnote{Shouldn't the probability over $\forall t\in [0, n/2]$. I think the order of the quantifier and the probability is wrong. Should be as in Lemma 5.3}\cnote{I corrected this error. Anupama, could you please confirm whether this is OK now?}\anote{Yes looks correct now.}%, satisfies: %\cnote{Please check the constants here. Considering that $\kappa=2e\log K$.}\anote{Changed the constants and also corrected the factor of $\log K$.} \cnote{Thanks! Looks good.}
$$\textstyle\|\bfd_t-\nabla F_{\cD}(x^t)\|_{*}  \leq 2[e\sqrt{2\log(K)}+\sqrt{3\log(n/\beta)}]
\Big(\frac{\sqrt{2}(1-\eta)^t L_0}{\sqrt{n}} + \eta \sqrt{t } \left( L_1 M + L_0\right)  \Big). $$
 \end{lem}

\else

\begin{lemma}\label{lem:privSFW_grad_var}
Let $\cD$ be any distribution over $\cZ$. Let $S\sim\cD^n$ be the input dataset of $\cA_\polyfw$ (\cref{Alg:PrivSFW}). For any $\beta \in(0,1)$, 
%the recursive gradient estimator $\bfd_t$  
with probability $\geq 1-\beta$, for all $t \in [0,\frac{n}{2}]$\rnote{Shouldn't the probability over $\forall t\in [0, n/2]$. I think the order of the quantifier and the probability is wrong. Should be as in Lemma 5.3}\cnote{I corrected this error. Anupama, could you please confirm whether this is OK now?}\anote{Yes looks correct now.}%, satisfies: %\cnote{Please check the constants here. Considering that $\kappa=2e\log K$.}\anote{Changed the constants and also corrected the factor of $\log K$.} \cnote{Thanks! Looks good.}
$$\textstyle\|\bfd_t-\nabla F_{\cD}(x^t)\|_{*}  \leq 2[e\sqrt{2\log(K)}+\sqrt{3\log(n/\beta)}]
\Big(\frac{\sqrt{2}(1-\eta)^t L_0}{\sqrt{n}} + \eta \sqrt{t } \left( L_1 M + L_0\right)  \Big). $$
 \end{lemma}

\fi
Now, the remaining ingredient in our analysis is to show that the noise added due to privacy in Step~\ref{stp:gamma} is bounded with high probability. 	

Let $\nu_t \triangleq \ip{v_t}{\bfd_{t}} - \min_{v \in \cV}{\ip{v}{\bfd_{t}}}$. The Report Noisy Max mechanism computes $\nu_t$ by finding the maximum over $K$ vertices of the convex polytope. Standard results for the tail bound of the maximum of  $K$ Laplace random variables, imply that for some $\beta'>0$, $\pr{}{\nu_t \geq \frac{2s_t\sqrt{n\log(1/\delta)}}{\varepsilon}\big(\log(\frac{K}{\beta'})\big)} \leq \beta'$. In particular, for $n/2$ iterations, we get obtain that the following event has probability $1-\beta$
\begin{align}
\textstyle \nu_t = O\left(\frac{s_t\sqrt{n\log(1/\delta)}}{\varepsilon}\log(\frac{Kn}{2\beta})\right) \quad \forall t \in [n/2]. \label{eqn:lap_noise_hp}
\end{align}
Finally we have all the components to state and prove the private and accuracy guarantees of Algorithm $\cA_\polyfw$.

\begin{theorem}%[Accuracy Guarantee of $\cA_\polyfw$]\label{thm:SFWAlgAcc}
	Let $\beta \in (0,1)$, and $\cD$ be any distribution over $\cZ$. Then, for the polyhedral setup, Algorithm~\ref{Alg:PrivSFW} is $(\varepsilon,\delta)$-DP and, with probability at least $1-\beta$, satisfies:
	%and let $S \sim \cD^n$. 
$$\Hp[\cA_\polyfw] =O\left( \frac{  L_1 M^2 + L_0M }{\varepsilon\,\sqrt{n}} \log\Big(\frac{n}{\log(K)}\Big) \log\Big(\frac{Kn}{\beta}\Big)\sqrt{\log(1/\delta)} \right).$$
%where $\Hp[\cA_\polyfw] = F_\cD({\cA_\polyfw}(S)) -  F_\cD^{\ast}.$
\end{theorem}
 
 \begin{proof}
	By smoothness and convexity of $F_\cD$:
	\begin{align*}
	\textstyle     F_\cD(x^{t+1}) & \textstyle \leq F_\cD(x^{t}) + \ip{\nabla F_\cD(x^t)}{x^{t+1} - x^t} + \frac{L_1}{2}\|x^{t+1} - x^t\|^2\\
	% &= F_\cD(x^{t}) + \eta \ip{\nabla F_\cD(x^t)}{v_t -x^t} + \frac{L_1 \eta^2 }{2}\|v_t - x^t\|^2 \qquad (x^{t+1}-x^t = \eta(v_t - x^t))\\
	&\textstyle \leq  F_\cD(x^{t}) +  \eta \ip{\nabla F_\cD(x^t) - \bfd_t}{v_t - x^t} + \frac{L_1\eta^2 M^2}{2} + \eta \ip{\bfd_t}{v_t -x^t}\\
	&\textstyle \leq F_\cD(x^{t}) + 2\eta M \dual{\nabla F_\cD(x^t) - \bfd_t} + \frac{L_1\eta^2 M^2}{2} + \eta \left(F_\cD(x^*) - F_\cD(x^t)\right) + \eta \nu_t.
	\end{align*}
	 Thus, we get
     $$ \textstyle F_\cD(x^{t+1})  - F_\cD(x^*) \leq (1-\eta)\left( F_\cD(x^t) - F_\cD(x^*)\right)+ 2\eta M \dual{\nabla F_\cD(x^t) - \bfd_t} + \frac{L_1\eta^2 M^2}{2} + \eta \nu_t. $$
	Letting $\G_t = F_\cD(x^{t})  - F_\cD(x^*)$, we get the following recursion:
		\begin{align}
 	\textstyle    \G_{t+1} \leq (1-\eta) \G_t + 2\eta M \,\dual{\nabla F_\cD(x^t) - \bfd_t} + \frac{L_1\eta^2 M^2}{2} +  \eta \nu_t. \nonumber
\end{align}
Note that, $t \in [0,\frac{n}{2}]$. By expanding the above recursion and using Lemma~\ref{lem:privSFW_grad_var} and (\ref{eqn:lap_noise_hp}), we have w.p. at least $1-\beta$ (taking $C_{\beta}:=e\sqrt{2\log(K)}+\sqrt{3\log(\frac{n}{\beta})}$ below):
	\begin{align*}
	\textstyle    \G_{t+1} &\leq (1-\eta)^{\frac{n}{2}+1} \G_0 + 4\eta M C_\beta
\Big(\frac{(1-\eta)^{\frac{n}{2}} L_0}{\sqrt{n}} +  \sqrt{n} \left( L_1 M + L_0\right)  \Big) + \frac{L_1\eta M^2}{2} \\
& +  \frac{2\log(\frac{Kn}{2\beta})\sqrt{n\log(1/\delta)}}{\varepsilon} \cdot \max\left\{(1-\eta)^{\frac{n}{2}}\,\frac{2L_0 M}{n}, 2\eta~( L_1M^2 + L_0 M) \right\}
	\end{align*}     
Choosing  $\eta = \frac{1}{n} \log\big(\frac{n}{\log(K)}\big)$, we get w.p. at least $1-\beta$:
	\begin{align*}
% 	\textstyle    \G_{t+1} &\leq L_0M \sqrt{\frac{\log(K)}{n}} + 8 \log\big(\frac{n}{\log(K)}\big)  [e\sqrt{e\log(K)}+\sqrt{3\log(n/\beta)}]
% \Big(\frac{\sqrt{\log(K)} L_0M}{n} +  \frac{ \left( L_1 M^2 + L_0M\right)}{\sqrt{n}}  \Big) \\
% &+ \frac{L_1  M^2 \log\big(\frac{n}{\log(K)}\big)}{2n} +  \frac{4\log(\frac{Kn}{\beta})\sqrt{\log(1/\delta)}}{\varepsilon} \cdot \max\left\{\,\frac{\sqrt{\log(K)}L_0 M}{n}, \frac{( L_1M^2 + L_0 M)}{\sqrt{n}} \right\}\\
\textstyle    \G_{t+1} &\leq L_0M \sqrt{\frac{\log(K)}{n}} + 4L_0M\sqrt{\log(K)}\log\left(\frac{n}{\log(K)}\right)\left(\frac{C_\beta}{n}  + \frac{\log(\frac{Kn}{\beta})\sqrt{\log(1/\delta)} }{\varepsilon~n}\right) \\
&\textstyle + 4( L_1M^2 + L_0 M)\log\left(\frac{n}{\log(K)}\right)\left(\frac{ C_\beta}{\sqrt{n}}+ \frac{\log(\frac{Kn}{\beta})\sqrt{\log(1/\delta)} }{\varepsilon\sqrt{n}} \right) \\
&\textstyle + 	 \frac{L_1 M^2}{2n}  \log\Big(\frac{n}{\log(K)}\Big). 
	\end{align*} 
	
\noindent By assuming $n > \log(K)$ (which is necessary to achieve non-trivial error even in the non-private setting), we obtain w.p. $\geq 1-\beta$:
$$\Hp[\cA_\polyfw] = O\left( \frac{  L_1 M^2 + L_0M }{\varepsilon\,\sqrt{n}} \log\left(\frac{n}{\log(K)}\right) \log\left(\frac{Kn}{\beta}\right)\sqrt{\log(1/\delta)} \right).$$
 \end{proof}

%% file: SCO_infty.tex
\section{Differentially Private SCO: $\ell_p$-setup for $2<p\leq \infty$}\label{sec:p_greater_2}
%SGD is Optimal for $2\leq p<+\infty$}

The proposed analyses, when applied to $\ell_p$-settings, only appear to provide useful bounds when $1\leq p\leq2$. This limitation comes from the fact that when $p>2$ the regularity constant of the dual, $\ell_{q}$, grows polynomially on the dimension, more precisely as $d^{1-2/p}$. This additional factor in the analysis substantially degrades the resulting excess risk bounds, unless $p\approx 2$. This leaves the question of what are the optimal rates for DP-SCO in such settings.

It is instructive to recall the optimal excess risk bounds for nonprivate SCO \cite{nemirovsky1983problem,Agarwal:2012}. These bounds have the form $\Theta(\min\{\frac{d^{1/2-1/p}}{\sqrt n},\frac{1}{n^{1/p}}\})$, and are attained by the combination of two different algorithms:
\begin{itemize}
    \item Stochastic Gradient Descent, for low dimensions $d\leq n$,
    %the low dimensional $d\leq n$ regime, 
    with rate $O(\frac{d^{1/2-1/p}}{\sqrt n})$.
    \item Stochastic Mirror Descent, for high dimensions $d>n$,  
    %(with regularizer $\frac1p\|x\|_p^p$), for the high dimensional $d>n$ regime, 
    with rate $O(\frac{1}{n^{1/p}})$.
\end{itemize}

%\color{red}
%\anote{Change to iterative localization and add remark about high probability bounds is possible}
In this section, we show that existing private stochastic gradient methods \cite{bassily2019private, feldman2020private,bassily2020stability, kulkarni2021private} that attain optimal excess risk in the Euclidean setting can also achieve nearly optimal excess risk in the low dimensional regime. For example, the \emph{phased stochastic gradient descent} \cite[Algorithm 2]{feldman2020private} attains the optimal excess population risk in Euclidean setting and runs in linear time. We will now show that in the low dimensional case, this algorithm attains nearly optimal excess risk in the $\ell_p$-setup where $2< p\leq \infty$. \ifarxiv\else The proof of this simple result is deferred to Appendix~\ref{appx:PSGD_p_geq_2}.\fi \rnote{I edited the above paragraph. There were wrong references and the phrasing was a bit confusing. I edited the proposition below as well}\cnote{I agree.} %for general losses the algorithm by \cite{kulkarni2021private}, and for smooth losses PNSGD \cite{feldman2019high} are optimal.

\begin{proposition}\label{prop:PSGD_p_geq_3}
	Consider the problem of DP-SCO in the $\ell_p=(\R^d,\|\cdot\|_p)$-setup, with $2< p \leq \infty$. The phased SGD algorithm ${\cal A}_{\mathsf{PhasedSGD}}$ \cite[Algorithm 2]{feldman2020private} runs in linear time and attains expected excess population risk 
	$$ \Risk_{\cal D}[{\cal A}_{\mathsf{PhasedSGD}}]=O\big(L_0M\big(\frac{d^{1/2-1/p}}{\sqrt{n}}+\frac{d^{1-1/p}\sqrt{\log(1/\delta)}}{\varepsilon n}\big)\big).$$
\end{proposition} 
\ifarxiv
\begin{proof}
We start by bounding the $\|\cdot\|_2$-diameter and Lipschitz constant for the $\ell_p$-setup. First, since the $\|\cdot\|_p$-diameter of ${\cal X}$ is bounded by $M$, then the $\|\cdot\|_2$-diameter of  ${\cal X}$ is bounded by $d^{1/2-1/p}M$. Next, if $f$ is $L_0$-Lipschitz w.r.t.~$\|\cdot\|_p$, i.e. $\bigcup_{x\in {\cal X}}\partial f(x)\subseteq {\cal B}_{\|\cdot\|_{p^{\ast}}}(0,L_0)$, then $f$ is also $L_0$-Lipschitz w.r.t.~$\|\cdot\|_2$. Therefore, \cite[Theorem 1.1]{feldman2020private} implies
%provides an expected excess risk bound for multipass noisy SGD of
\begin{eqnarray*} 
\Risk_{\cal D}[{\cal A}_{\nsgd}] &=& \textstyle d^{\frac12-\frac1p} L_0M\cdot O\Big(\max\Big\{\frac{1}{\sqrt n},\frac{\sqrt{d\log(1/\delta)}}{\varepsilon n}\Big\} \Big)
\\&=& \textstyle 
L_0M \cdot O\Big(\frac{d^{\frac12-\frac1p}}{\sqrt n}+\frac{d^{1-\frac1p}\sqrt{\log(1/\delta)}}{\varepsilon n} \Big).
\end{eqnarray*}
\end{proof}
\fi

We conclude this section observing that in the low-dimensional regime, i.e., $d\lesssim n$, the above upper bound is optimal since it matches the optimal non-private lower bound of $\Omega\left(\frac{d^{1/2-1/p}}{\sqrt n}\right)$ \cite{Agarwal:2012}. Note that in the $\ell_{\infty}$ setting, the low-dimensional regime is the only interesting regime since the excess risk is $\Omega(1)$ if $d>n$. Hence, our result implies that phased SGD \cite{feldman2020private} attains nearly optimal excess risk for DP-SCO in the $\ell_{\infty}$ setting. We formally state this observation below. %This is a direct consequence of the following result.

% We conclude this section observing that in the $\ell_{\infty}$ setting, SGD attains the optimal excess risk for DP-SCO in the interesting regime when $d\lesssim n$ (notice that if $d>n$ then the excess risk is $\varnothingmega(1)$). This is a direct consequence of the following result.

\begin{corollary} \label{cor:SGD_optimal_infty}
	Let $2<p\leq \infty$ and ${\cal X}={\cal B}_{\|\cdot\|_p}(0,M)$. Let $d\log(1/\delta)/\varepsilon^2\leq n$. The phased SGD algorithm ${\cal A}_{\mathsf{PhasedSGD}}$ \cite[Algorithm 2]{feldman2020private} runs in linear time and attains the optimal excess population risk for DP-SCO in the $\ell_{p}$-setup. For $p=\infty$, this algorithm attains nearly optimal excess population risk. %\rnote{made minor rephrasing} \cnote{Looks good to me.}
\end{corollary}
\rnote{added the remark below. I also made several edits in the remainder of the section.}
\begin{remark}
Our results here are based on the excess risk guarantees of \cite{feldman2020private}, which are expectation guarantees rather than high-probability guarantees. However, we note that it is possible to provide high-probability guarantees for the phased SGD algorithm ${\cal A}_{\mathsf{PhasedSGD}}$ \cite[Algorithm 2]{feldman2020private} by slightly modifying the algorithm (where the mini-batch size in each phase is larger by a logarithmic factor) together with a more careful analysis that uses a standard martingale argument to show high-probability convergence guarantee for the projected SGD invoked in each phase of ${\cal A}_{\mathsf{PhasedSGD}}$. 
\end{remark} 

Note that, for general losses, Kulkarni et. al. \cite{kulkarni2021private} give an algorithm for DP-SCO in the Euclidean setup that attains the optimal excess population risk. Their algorithm runs in $O(\min\{n^{5/4}d^{1/8},\frac{n^{3/2}}{d^{1/8}}\})$ time. This algorithm combines a DP-ERM algorithm \cite[Algorithm 4]{kulkarni2021private} with an iterative localization approach from \cite{feldman2020private}. Following the same argument as in the smooth case, we have the following result for general convex losses.
\begin{proposition}\label{prop:SGD_p_geq_2}
	For general convex losses, consider the problem of DP-SCO in the $\ell_p=(\R^d,\|\cdot\|_p)$-setup, with $2< p \leq \infty$. Let $\cA_\kul$ denote the algorithm proposed in \cite{kulkarni2021private}. Then, $\cA_\kul$ attains expected excess population risk 
	$$\Risk_{\cal D}[\cA_\kul]= O\big(L_0M\big(\frac{d^{1/2-1/p}}{\sqrt{n}}+\frac{d^{1-1/p}\sqrt{\log(1/\delta)}}{\varepsilon n}\big)\big)$$ 
	and runs in $O(\min\{n^{5/4}d^{1/8},n^{3/2}/d^{1/8}\})$ time.
\end{proposition}
Moreover, as in Corollary~\ref{cor:SGD_optimal_infty}, when $d\log(1/\delta)/\varepsilon^2\leq n$, this algorithm attains the optimal excess risk for any $\ell_p$ setup, where $p>2$. Hence, when $p=\infty$, this algorithm attains nearly optimal excess risk.

\color{black}

%% file: ack.tex
\section*{Acknowledgements}\label{sec:ack}
RB's and AN's research is supported by NSF Award AF-1908281, NSF Award 2112471, Google Faculty Research Award, and the OSU faculty start-up support. 
CG’s research is partially supported by INRIA through the INRIA Associate Teams project and FONDECYT 1210362 project.

%% file: arxiv_appendix.tex
\appendix
\section{Lower Bound for the $\ell_p$ Setup for $1< p < 2$}\label{appx:lower_bound}

In this section, we give lower bounds on the excess risk of DP-SCO and DP-ERM in the $\ell_p$ setting for $1<p<2$. %population risk of any $(\varepsilon, \delta)$ DP-SCO algorithm in the $\ell_p$ setting for $1<p<2$. 
Our lower bound for DP-SCO has the form 
$\tilde\Omega\big(\max\big(\frac{1}{\sqrt{n}}, (p-1)\frac{\sqrt{d}}{\varepsilon n}\big)\big)$.
%$\tilde\Omega\Big(\max\Big(1/\sqrt{n}, (p-1)\sqrt{d}/(\varepsilon n)\Big)\Big)$. 
The first term follows directly from the non-private lower bound for SCO in the same setting \cite{nemirovsky1983problem}. To establish a lower bound 
with the second term, 
%of $\tilde\Omega((p-1)\sqrt{d}/(\varepsilon n))$, 
we show a lower bound of essentially the same order (up to a logarithmic factor in $1/\delta$) on the excess empirical error for DP-ERM in the $\ell_p$ setup (Theorem~\ref{thm:ERM_lower_bound}). %Given 
By the reduction in \cite[Appendix C]{bassily2019private}, we conclude %this implies 
the claimed lower bound on DP-SCO. 

\mypar{Problem setup:} Let $p \in (1, 2)$ and $d\in \mathbb{N}$. Let $\cX=\B_p^d$, where $\B_p^d$ is the unit $\ell_p$ ball in $\R^d$, and let $\cZ=\{-\frac{1}{d^{1/q}}, \frac{1}{d^{1/q}} \}^d$ where $q=\frac{p}{p-1}$. Let $f:\cX\times \cZ \rightarrow [-1, 1]$ defined as:
\begin{align*}
  \textstyle  f(x, z)=-\langle x, z\rangle,~ x\in \cX, z\in\cZ.
\end{align*}
Note that for every $z\in\cZ$, $f(\cdot, z)$ is convex, smooth, and $1$-Lipschitz w.r.t. $\|\cdot\|_p$ over $\cX$. Recall that for any distribution $\cD$ over $\cZ$, we define the population risk of $x\in\cX$ w.r.t. $\cD$ as $F_{\cD}(x)\triangleq \ex{z\sim\cD}{f(x, z)}$, and for any dataset $S=(z_1, \ldots, z_n)\in \cZ^n,$ we define the empirical risk of $x\in\cX$ w.r.t. $S$ as $F_S(x)\triangleq \frac{1}{n}\sum_{i=1}^n f(x, z_i)$.   

\vspace{0.2cm}

Our lower bound for DP-SCO is formally stated in the following theorem. 

\begin{theorem}\label{thm:lower_bound}
Let $p \in (1, 2)$ and $n, d\in \mathbb{N}$. Let $\varepsilon >0$ and $0< \delta <\frac{1}{n^{1+{\Omega(1)}}}$. Let $\cX, \cZ,$ and $f$ be as defined in the setup above. There exists a distribution $\cD$ over $\cZ$ such that for any $(\varepsilon, \delta)$-DP-SCO algorithm $\cA:\cZ^n\rightarrow \cX$, we have 
$$\textstyle \ex{S\sim\cD^n, \cA}{F_{\cD}(\cA(S))}-\min\limits_{x\in\cX}F_{\cD}(x)=\tilde\Omega\Big(\max\Big(\frac{1}{\sqrt{n}}, (p-1)\frac{\sqrt{d}}{\varepsilon n}\Big)\Big).$$
\end{theorem}

As mentioned earlier, given the reduction from \cite{bassily2019private}, it suffices to prove a lower bound of essentially the same order for DP-ERM w.r.t.~the problem described above. The rest of this section will be devoted to this goal; namely, to prove the following. %theorem.

\begin{theorem}\label{thm:ERM_lower_bound}
Under the same setup in Theorem~\ref{thm:lower_bound}, there exists a dataset $S\in \cZ^n$ such that for any $(\varepsilon, \delta)$-DP-ERM algorithm $\cA:\cZ^n\rightarrow \cX$, we have 
$$\textstyle \ex{\cA}{F_{S}(\cA(S))}-\min\limits_{x\in\cX}F_{S}(x)=\Omega\Big((p-1)\frac{\sqrt{d\log(1/\delta)}}{\varepsilon n}\Big).$$
\end{theorem}   

\begin{proof}
Let the spaces $\cX, \cZ,$ and the loss function $f$ be as defined in the problem setup above. For any $x\in \cX$, let $x_j$ denote the $j$-th coordinate of $x$, where $j\in [d]$. Let $S=(z_1,\ldots, z_n)\in\cZ^n$, and let $z_{ij}$ denote the $j$-th coordinate of $z_i$, where $i\in [n], j\in [d]$. Define $\brz\triangleq \frac{1}{n}\sum_{i=1}^n z_i$, and similarly, let $\brz_j$ denote the $j$-th coordinate of $\brz$. 

%and let $v_{ij}\triangleq \mathsf{sign}(z_{ij}),$

Let $\xs\triangleq \arg\min\limits_{x\in\cX}F_S(x)=\arg\max\limits_{x\in\cX}\langle x, \brz\rangle$. Note that we have  
\begin{align}\label{eqn:opt_x}
 \xs_j &= \textstyle \big(|\brz_j|^{q-1}/\|\brz\|_q^{q-1}\big)\mathsf{sign}(\brz_j), \quad j\in [d],
%    \xs_j &= \frac{|\brz_j|^{q-1}}{\|\brz\|_q^{q-1}}\mathsf{sign}(\brz_j), \quad j\in [d]
\end{align}
where $\|\cdot\|_q$ denote the $\ell_q$ norm (recall that $q\triangleq \frac{p}{p-1}$). To see this, note that by Hölder's inequality $\forall x\in \cX,$ $F_S(x)\geq -\|\brz\|_q$, and on the other hand, note that $\xs\in \cX$ since $\|\xs\|_p=1$ and $F_S(\xs)=-\|\brz\|_q$. 
Next, we make the following claim. 
\begin{claim}\label{claim:risk_to_distance}
Let $\alpha>0$. Let $\hx\in \cX$ be such that $F_S(\hx)-F_S(\xs)\leq \alpha$. Then, $\|\hx-\xs\|_p\leq \sqrt{\frac{8 \alpha}{(p-1)\|\brz\|_q}}.$
\end{claim}
The proof of this claim relies on the uniform convexity property of the $\ell_p$ norms for $p\in (1, 2]$ (see \cite{ball1994sharp}). We formally restate this property below:
\begin{fact}[see Eq. (1.6) in \cite{ball1994sharp}]
Let $x, y$ be any elements of an $\ell_p$-normed space \mbox{$(\bE, \|\cdot\|_p)$,} where $1< p \leq 2$. We have $\|\frac{x+y}{2}\|_p \leq 1-\frac{p-1}{8}\|x-y\|_p^2$.  
\end{fact}
Now, observe that for any $\hx \in \cX$ such that $F_S(\hx)-F_S(\xs)\leq \alpha$, we have
\begin{align*}
    \textstyle 1-\frac{\alpha}{2\|\brz\|_q}& \textstyle \leq\langle \frac{\hx+\xs}{2}, \frac{\brz}{\|\brz\|_q} \rangle\leq \|\frac{\hx+\xs}{2}\|_p\leq 1-\frac{p-1}{8}\|\hx-\xs\|_p^2,
\end{align*}
where the last inequality follows from the above fact. Rearranging terms leads to the above claim.

Fix values for $\varepsilon$ and $\delta$ as in the theorem statement. Next, we will show the existence of a dataset $S=(z_1, \ldots, z_n)\in\cZ^n$ with $\|\brz\|_q=\Omega\Big(\frac{\sqrt{d\log(1/\delta)}}{\varepsilon n}\Big)$ such that for any $(\varepsilon, \delta)$-DP-ERM algorithm for the above problem that outputs a vector $\hx\in \cX$, we must have $\|\hx-\xs\|_p=\Omega(1)$ with probability $2/3$ over the algorithm's random coins. Note that, by Claim~\ref{claim:risk_to_distance}, this implies the desired lower bound. To see this, suppose, for the sake of a contradiction, that there exists an $(\varepsilon, \delta)$-DP-ERM algorithm $\cA$ that outputs $\hx\in\cX$ such that $\ex{\hx\leftarrow \cA}{F_S(\hx)}-F_S(\xs)=o\Big((p-1)\frac{\sqrt{d\log(1/\delta)}}{\varepsilon n}\Big)$. Then, by Markov's inequality, with probability $\geq 0.9,$ we have $F_S(\hx)-F_S(\xs)=o\Big((p-1)\frac{\sqrt{d\log(1/\delta)}}{\varepsilon n}\Big)$. Hence, Claim~\ref{claim:risk_to_distance} would imply that, with probability $\geq 0.9$, $\|\hx-\xs\|_p=o(1)$, which contradicts with the claimed $\Omega(1)$ lower bound on $\|\hx-\xs\|$. Hence, to conclude the proof of Theorem~\ref{thm:ERM_lower_bound}, it remains to show the claimed lower bound on $\|\hx-\xs\|$, which we do next.  

In the final step of the proof, we resort to a construction based on the fingerprinting code argument due to \cite{bun2018fingerprinting}. We use the following lemma, which is implicit in the constructions of \cite{bun2015differentially, steinke2015between}. 

\ifarxiv
\begin{lem}\label{lem:fingerprint}
Let $n, d\in\mathbb{N}$. Let $\varepsilon >0$ and $0< \delta <\frac{1}{n^{1+{\Omega(1)}}}$. Let $\cT=\{-1, 1\}^d$. There exists a dataset $T=(v_1, \ldots, v_n)\in \cT^n$ where $\|\frac{1}{n}\sum_{i=1}^n v_i\|_{\infty}\leq c\frac{\sqrt{d\log(1/\delta)}}{\varepsilon n}$ for some universal constant $c>0$ such that for any $(\varepsilon, \delta)$-differentially private algorithm $\cM:\cT^n\rightarrow [-1, 1]^d,$ the following is true with probability $2/3$ over the random coins of $\cM$: $\exists J \subseteq [d]$ with $|J|=\Omega(d)$ such that
$$(\forall j\in J),\quad \Big\lvert\cM_j(T)-\frac{1}{n}\sum_{i=1}^n v_{ij}\Big\rvert=\Omega\Big(\frac{\sqrt{d\log(1/\delta)}}{\varepsilon n}\Big) \quad\text{and}\quad \Big\lvert\sum_{i=1}^n v_{ij}\Big\rvert=c\cdot\frac{\sqrt{d\log(1/\delta)}}{\varepsilon n},$$ 
where $\cM_j(T)$ and $v_{ij}$ denote the $j$-th coordinates of $\cM(T)$ and $v_i$, respectively.
%where $\cM_j(T)$ denotes the $j$-th coordinate of $\cM(T)$ and $v_{ij}$ denotes the $j$-th coordinate of $v_i$. 
\end{lem}

\else

\begin{lemma}\label{lem:fingerprint}
Let $n, d\in\mathbb{N}$. Let $\varepsilon >0$ and $0< \delta <\frac{1}{n^{1+{\Omega(1)}}}$. Let $\cT=\{-1, 1\}^d$. There exists a dataset $T=(v_1, \ldots, v_n)\in \cT^n$ where $\|\frac{1}{n}\sum_{i=1}^n v_i\|_{\infty}\leq c\frac{\sqrt{d\log(1/\delta)}}{\varepsilon n}$ for some universal constant $c>0$ such that for any $(\varepsilon, \delta)$-differentially private algorithm $\cM:\cT^n\rightarrow [-1, 1]^d,$ the following is true with probability $2/3$ over the random coins of $\cM$: $\exists J \subseteq [d]$ with $|J|=\Omega(d)$ such that
$$(\forall j\in J),\quad \Big\lvert\cM_j(T)-\frac{1}{n}\sum_{i=1}^n v_{ij}\Big\rvert=\Omega\Big(\frac{\sqrt{d\log(1/\delta)}}{\varepsilon n}\Big) \quad\text{and}\quad \Big\lvert\sum_{i=1}^n v_{ij}\Big\rvert=c\cdot\frac{\sqrt{d\log(1/\delta)}}{\varepsilon n},$$ 
where $\cM_j(T)$ and $v_{ij}$ denote the $j$-th coordinates of $\cM(T)$ and $v_i$, respectively.
%where $\cM_j(T)$ denotes the $j$-th coordinate of $\cM(T)$ and $v_{ij}$ denotes the $j$-th coordinate of $v_i$. 
\end{lemma}

\fi
We consider a normalized version of the dataset $T=(v_1, \ldots, v_n)$ in the above lemma. Namely, we consider a dataset $S=(z_1, \ldots, z_n)\in\cZ^n$, where $z_i=\frac{v_i}{d^{1/q}},~ i\in[n].$ Note that the above lemma implies the existence of a subset $J\subseteq [d]$ with $|J|=\Omega(d)$ such that for all $j\in J,$ $|\brz_j|=\frac{c}{d^{1/q}}\cdot \frac{\sqrt{d\log(1/\delta)}}{\varepsilon n}$ for some universal constant $c>0$. Note also that  $\forall~j\in [d]\setminus J,$~ $|\brz_j|\leq \frac{c}{d^{1/q}}\cdot \frac{\sqrt{d\log(1/\delta)}}{\varepsilon n}$ since $\|\frac{1}{n}\sum_{i=1}^n v_i\|_{\infty}\leq c\frac{\sqrt{d\log(1/\delta)}}{\varepsilon n}$. This implies that 
\begin{align*}
  \frac{|J|}{d}\,c^q\left(\frac{\sqrt{d\log(1/\delta)}}{\varepsilon n}\right)^q\leq\|\brz\|_q^q&\leq c^q\left(\frac{\sqrt{d\log(1/\delta)}}{\varepsilon n}\right)^q,
\end{align*}
which, given the fact that $|J|=\Omega(d)$, implies that $\|\brz\|^{q}_q=c'\frac{\sqrt{d\log(1/\delta)}}{\varepsilon n}$ for some universal constant $c'>0$. Hence, by the fact that $q>2$, we have $\|\brz\|_q=\Theta\big(\frac{\sqrt{d\log(1/\delta)}}{\varepsilon n}\big)$. Moreover, note that for all $j\in [J]$, we have $\frac{|\brz_j|^{q-1}}{\|\brz\|_q^{q-1}}=\frac{\left(\frac{c}{c'}\right)^{1-\frac{1}{q}}}{d^{1-\frac{1}{q}}}=\frac{c''}{d^{1/p}}$ for some universal constant $c''$, where the last equality follows from the fact that $q>2$ and $\frac{1}{q}=1-\frac{1}{p}$.

Let $\brv\triangleq \frac{1}{n}\sum_{i=1}^n v_i$, and let $\brv_j$ denote the $j$-th coordinate of $\brv$ for $j\in [d]$. Given the above observations and the expression of the minimizer $\xs$ in eq.~\ref{eqn:opt_x}, it is not hard to see that for all $j\in J,$
\begin{align}\label{eqn:asym-exp-xs}
\xs_j&=\frac{c''}{d^{1/p}}\mathsf{sign}(\brv_j)=\frac{c''}{d^{1/p}}\cdot\frac{\brv_j}{|\brv_j|}=\frac{c''}{c}\cdot\frac{\varepsilon n}{d^{1/2+1/p}\sqrt{\log(1/\delta)}}\brv_j.
\end{align}

Let $\cA$ be any $(\varepsilon, \delta)$-DP-ERM algorithm that takes the dataset $S$ described above as input, and let $\hx\in\cX$ denote its output. Construct an $(\varepsilon, \delta)$-differentially private algorithm $\cM$ for the dataset $T$ of Lemma~\ref{lem:fingerprint} by first running $\cA$ on $S=\frac{1}{d^{1/q}}\cdot T$, which outputs $\hx$, then releasing $\cM(T)=\frac{c}{c''}\cdot\frac{d^{1/2+1/p}\sqrt{\log(1/\delta)}}{\varepsilon\,n}\cdot\hx$. Now, using \ref{eqn:asym-exp-xs} and given the description of $\cM$, observe that
\begin{align*}
    \textstyle\|\hx-\xs\|_p
    %&\textstyle =\frac{c''}{c}\cdot\frac{\varepsilon n}{d^{1/2+1/p}\sqrt{\log(1/\delta)}}\cdot\|\cM(T)-\brv\|_p\\
    &\textstyle 
    \geq\frac{c''}{c}\cdot\frac{\varepsilon n}{d^{1/2+1/p}\sqrt{\log(1/\delta)}}\cdot\Big(\sum_{j\in J}|\cM_j(T)-\brv_j|^p\Big)^{1/p}\\
    &\textstyle =\Omega\left(\frac{\varepsilon n}{d^{1/2+1/p}\sqrt{\log(1/\delta)}}\,d^{1/p}\,\frac{\sqrt{d\log(1/\delta)}}{\varepsilon\,n}\right)\\
    &=\Omega(1)
\end{align*}
where the third step follows from Lemma~\ref{lem:fingerprint} and the fact that $\cM$ is $(\varepsilon, \delta)$-DP. This establishes the desired lower bound on $\|\hx-\xs\|_p$, and hence by the argument described earlier, the proof of Theorem~\ref{thm:ERM_lower_bound} is now complete. 
\end{proof}